\documentclass[12pt,pagebackref=true]{article}
\usepackage{fullpage}
\usepackage[round]{natbib}

\usepackage{caption}
\usepackage{microtype}
\usepackage{graphicx}
\usepackage{subfigure}
\usepackage{booktabs}

\usepackage{hyperref}

\usepackage{amsmath}
\usepackage{amssymb}
\usepackage{mathtools}
\usepackage{amsthm}

\usepackage{dirtytalk}
\usepackage[capitalize,noabbrev]{cleveref}
\usepackage{titletoc}
\usepackage{wrapfig}
\usepackage{xspace}

\usepackage{xcolor}  

\usepackage{multirow}
\usepackage[cp1250]{inputenc}
\usepackage[T1]{fontenc}
\usepackage{calligra}
\usepackage{layout}

\usepackage{amsmath}
\usepackage{amsthm}
\usepackage{amssymb}
\usepackage{amsfonts}
\usepackage{mathtools}

\usepackage{url}
\usepackage{color}
\usepackage{graphicx}
\usepackage{algorithmic}
\usepackage{algorithm}
\usepackage{verbatim}
\usepackage{appendix}

\newcommand{\Clip}[2]{{\rm Clip}_{#1}\left( #2 \right)}
\newcommand{\Normalize}[2]{{\rm Norm}_{#1}\left( #2 \right)}

\newcommand{\norm}[1]{\left\| #1 \right\|}
\newcommand{\sqnorm}[1]{\left\| #1 \right\|^2}
\newcommand{\inp}[2]{\left\langle#1,#2\right\rangle} 

\newcommand\br[1]{\left( #1 \right)}

\newcommand{\cB}{\mathcal{B}}
\newcommand{\cC}{\mathcal{C}}

\newcommand{\cF}{\mathcal{F}}

\newcommand{\cO}{\mathcal{O}}

\newcommand{\del}[1]{}

\definecolor{junglegreen}{rgb}{0.16, 0.67, 0.53}
\definecolor{lasallegreen}{rgb}{0.03, 0.47, 0.19}
\definecolor{midnightblue}{HTML}{0059b3}

\newcommand{\algname}[1]{\textcolor{midnightblue!70!black}{\small\sf#1}}

\newcommand{\R}{\mathbb{R}} 

\newcommand{\eqdef}{:=} 

\newcommand{\Exp}[1]{{\rm E}\left[#1\right]}
\newcommand{\ExpCond}[2]{{\rm E}\left[\left.#1\right\vert#2\right]}

\newtheorem{assumption}{Assumption}
\newtheorem{lemma}{Lemma}

\newtheorem{theorem}{Theorem}

\newtheorem{example}{Example}
\newtheorem{corollary}{Corollary}

\theoremstyle{plain}

\theoremstyle{definition}

\definecolor{red}{rgb}{1.0, 0.01, 0.24}

\usepackage[textsize=tiny]{todonotes}

\definecolor{darkscarlet}{rgb}{0.34, 0.01, 0.1}
\definecolor{yaleblue}{rgb}{0.06, 0.3, 0.57}
\definecolor{darkpowderblue}{rgb}{0.0, 0.2, 0.6}
\definecolor{midnightblue}{HTML}{0059b3}
\definecolor{noonblue}{HTML}{e5eef7}
\definecolor{chromered}{HTML}{f14233}
\definecolor{olivedrab}{HTML}{6b8e23}

\hypersetup{
    colorlinks=true,
    citecolor=darkscarlet,
    linkcolor=darkpowderblue,
    filecolor=magenta,      
    urlcolor=yaleblue,
    menucolor=gray
}

\renewcommand*{\backrefalt}[4]{%
    \ifcase #1 \footnotesize{(Not cited.)}%
    \or        \footnotesize{(Cited on page~#2)}%
    \else      \footnotesize{(Cited on pages~#2)}%
    \fi}

\newcommand{\alg}{\textcolor{midnightblue!70!black}{\small\sf $\alpha$-NormEC}\xspace}

\usepackage{subcaption}

\makeatletter
\renewcommand{\paragraph}[1]{\vspace{4pt}\textbf{#1}\quad} 
\makeatother

\title{\textbf{Smoothed Normalization for Efficient \\  Distributed Private Optimization}}

\usepackage{authblk}

\author{
Egor Shulgin \qquad Sarit Khirirat \qquad Peter Richt{\'a}rik \\
\phantom{x}
    \\
    King Abdullah University of Science and Technology (KAUST) \\
    Thuwal, Saudi Arabia
}

\date{}

\begin{document}

\maketitle

\begin{abstract}
Federated learning enables training machine learning models while preserving the privacy of participants. Surprisingly, there is no differentially private distributed method for smooth, non-convex optimization problems. The reason is that standard privacy techniques require bounding the participants' contributions, usually enforced via \textit{clipping} of the updates. Existing literature typically ignores the effect of clipping by assuming the boundedness of gradient norms or analyzes distributed algorithms with clipping but ignores DP constraints. In this work, we study an alternative approach via \textit{smoothed normalization} of the updates motivated by its favorable performance in the single-node setting. By integrating smoothed normalization with an error-feedback mechanism, we design a new distributed algorithm \alg. We prove that our method achieves a superior convergence rate over prior works. By extending \alg to the DP setting, we obtain the first differentially private distributed optimization algorithm with provable convergence guarantees.
Finally, our empirical results from neural network training indicate robust convergence of \alg across different parameter settings.
\end{abstract}

\section{Introduction}\label{sec:intro}

Federated Learning (FL) has become a viable approach for distributed collaborative 
training of machine learning models \citep{konecny2017federated, mcmahan2017communication, mcmahan2018learning}.  
This growing interest has spurred the development of novel distributed optimization methods tailored for FL, focusing on 
ensuring high \textit{communication efficiency}~\citep{kairouz2021advances}.
Although FL optimization methods ensure that private data is never directly transmitted, \citet{boenisch2023curious} demonstrated that the global models produced through FL can still enable the reconstruction of participants' data.
Therefore, it is essential to study \emph{differentially private}  distributed optimization methods for differentially private training \citep{dwork2014algorithmic, mcmahan2018learning,sun2019can}.

To mitigate emerging privacy risks in FL, differential privacy (DP)~\citep{dwork2014algorithmic} has become the standard for providing theoretical privacy guarantees in machine learning. 
DP is often enforced by a clipping operator. 
It bounds gradient sensitivity, allowing the addition of DP noise to the updates before communication.
While gradient clipping enables DP as in Differentially Private Stochastic Gradient Descent (\algname{DP-SGD})~\citep{abadi2016deep}, it also introduces a bias that can impede convergence
~\citep{chen2020understanding, koloskova2023revisiting}. 
Often, distributed DP gradient methods with clipping have been studied under assumptions that are unrealistic for heterogeneous FL environments, such as 
bounded gradient norms~\citep{li2022soteriafl,pmlr-v216-wang23b,lowy2023private,zhang2020private}, which effectively ignore the impact of clipping bias.
To our knowledge, convergence guarantees for distributed DP methods remain elusive unless the impact of clipping bias is explicitly considered.

Error Feedback (EF), also known as Error Compensation (EC), such as \algname{EF21} \citep{richtarik2021ef21} has been employed to alleviate the clipping bias and achieve strong convergence for non-private distributed methods with gradient clipping, as shown by~\citet{khirirat2023clip21, yu2023smoothed}.
However, extending these methods to the private setting remains an open problem. 
Furthermore, optimizing the convergence of distributed DP clipping methods is challenging because the clipping threshold significantly influences both the convergence speed and  DP noise variance.
Extensive grid search for the optimal clipping threshold is computationally expensive~\citep{andrew2021differentially} and leads to additional privacy loss~\citep{papernot2021hyperparameter}. Two major approaches have emerged to address the need to manually tune the clipping threshold. The first is to use adaptive clipping techniques, such as adaptive quantile clipping, initially proposed by \citet{andrew2021differentially} and further analyzed by \citet{merad2023robust,shulgin2024convergence}. The second, which is the focus of this paper, is to replace clipping with a normalization operator.

{Smoothed normalization}, introduced by \citet{bu2024automatic, yang2022normalized}, is the normalization operator that offers an alternative to clipping.
Unlike clipping, smoothed normalization eliminates the need to tune the clipping threshold. By ensuring that the Euclidean norm of the normalized gradient is bounded above by one, smoothed normalization guarantees robust performance of \algname{DP-SGD} in convergence and privacy. 
However, very limited literature characterizes properties of smoothed normalization and a rigorous convergence analysis for \algname{DP-SGD} using this operator, {especially in the distributed setting}. While the method has been studied {in the single-node setting} by \citet{bu2024automatic} and \citet{yang2022normalized}, the convergence results rely on {unrealistic and/or restrictive} assumptions, such as symmetric gradient noise~\citep{bu2024automatic} and almost sure bounds on the gradient noise variance~\citep{yang2022normalized}.

\subsection{Contributions}
We propose \alg, the distributed gradient method that uses smoothed normalization and error compensation.
Our method provides the first provable convergence guarantees in the DP setting without bounded gradient norm assumptions typically imposed in prior works.  
Our contributions are summarized below:

\paragraph{$\bullet$ Favorable properties of smoothed normalization.}
In Section~\ref{sec:not_def}, we present the novel properties of smoothed normalization. 
We show that smoothed normalization enjoys a \say{contractive} property similar to biased compression operators \citep{beznosikov2023biased} widely used for reducing communication in distributed learning. This property essentially allows for analyzing \alg without ignoring the impact of smoothed normalization. 

\paragraph{$\bullet$ Convergence for non-convex, smooth problems without bounded gradient norm assumptions.} 
In Section~\ref{sec:norm21}, we prove that 
\alg achieves the optimal convergence rate \citep{carmon2020lower} 
for minimizing non-convex, smooth functions without imposing additional restrictive assumptions, such as bounded gradient norms or bounded heterogeneity.
Furthermore, \alg achieves a faster convergence rate than  \algname{Clip21}~\citep{khirirat2023clip21}, where its step size needs to know the inaccessible value of $f(x^0)-f^{\inf}$.

\paragraph{$\bullet$ 
The first provable convergence in the private setting under standard assumptions.}
In Section~\ref{sec:dp_norm21}, we extend \alg to the differential privacy (DP) setting. 
Specifically, \alg achieves the first convergence guarantees for DP, non-convex, smooth problems \emph{without} ignoring the bias introduced by smoothed normalization. 
This is the first provably efficient distributed method in the DP setting under standard assumptions, thus addressing the theoretical gap left by prior work such as by \citet{khirirat2023clip21,yu2023smoothed}, which did not adapt distributed gradient clipping methods for private training.

\paragraph{$\bullet$ Robust empirical performance  of \alg.} 
In Section~\ref{sec:exp}, we verify the theoretical benefits of \alg in both non-private and private training via experiments on the image classification task with the CIFAR-10 dataset using the ResNet20 model. 
Our algorithm demonstrates robust convergence across different parameter values and benefits from error compensation that enables superior performance over vanilla distributed gradient normalization methods (such as  \algname{DP-SGD}).
In the private training, server normalization enhances the robustness of \algname{DP-}\alg across tuning parameters. Finally, \algname{DP-}\alg without server normalization outperforms \algname{DP-Clip21}.

\section{Related Work}\label{sec:related_work}

\paragraph{Clipping and normalization.} 
Clipping and normalization
address many key challenges in machine learning. 
They mitigate the problem of exploding gradients in recurrent neural networks \citep{pascanu2013difficulty}, enhance neural network training for tasks in natural language processing \citep{merity2017regularizing,brown2020language} and computer vision \citep{brock2021high}, enable differentially private machine learning \citep{abadi2016deep,mcmahan2018learning}, and provide robustness in the presence of misbehaving or adversarial workers \citep{karimireddy2021learning,ozfatura2023byzantines,malinovsky2023byzantine}. 
In this paper, we consider smoothed normalization, introduced by \citet{yang2022normalized, bu2024automatic}, as an alternative to clipping, given its robust empirical performance and hyperparameter tuning benefits in the DP setting.

\paragraph{Private optimization methods.} 
\algname{DP-SGD}~\citep{abadi2016deep} is the standard distributed first-order method that achieves the DP guarantee by clipping  the gradient
before adding noise scaled with the clipped gradient's sensitivity. 
However, existing \algname{DP-SGD} convergence analyses often neglect the clipping bias. 
Specifically, convergence results for smooth functions under differential privacy often require either the assumption of bounded gradient norms~\citep{zhang2020private,li2022soteriafl,zhang2022understanding,pmlr-v216-wang23b,lowy2023private, murata2023diff2, wang2024efficient} or conditions where clipping is effectively inactive~\citep{zhang2024private,noble2022differentially}. 
Thus, the convergence behavior of \algname{DP-SGD} in the presence of clipping bias remains poorly understood.

\paragraph{Single-node non-private methods with clipping. } 
The impact of clipping bias has been extensively studied in single-node gradient methods for non-private optimization. 
Numerous works have shown strong convergence guarantees of clipped gradient methods under various conditions, including nonsmooth, rapidly growing convex functions~\citet{shor2012minimization,ermoliev1988stochastic,alber1998projected}, generalized smoothness~\citep{zhang2019gradient,koloskova2023revisiting,gorbunov2024methods,vankov2024optimizing,lobanov2024linear,hubler2024parameter}, and heavy-tailed noise~\citep{gorbunov2020stochastic,nguyen2023improved,gorbunov2024highprobability,hubler2024gradient,chezhegov2024gradient}.

\paragraph{Distributed non-private methods with clipping.}
Applying gradient clipping in the distributed setting is challenging.
Existing convergence analyses often rely on bounded heterogeneity assumptions, which often do not hold in cases of arbitrary data heterogeneity.
For example, federated optimization methods with clipping have been analyzed under the bounded difference between the local and global gradients~\citep{wei2020federated,liu2022communication,crawshaw2023episode,li2024an}.
However, even in the non-private setting, these distributed clipping methods do not converge for simple problems~\citep{chen2020understanding,khirirat2023clip21}. 
To address the convergence issue, one approach is to use error feedback mechanisms, such as \algname{EF21} \citep{richtarik2021ef21}, that are employed by \citet{khirirat2023clip21,yu2023smoothed} to compute local gradient estimators and alleviate clipping bias. However, these distributed clipping methods using error feedback are limited to the non-private setting, and extending them to the DP setting is still an open problem.
In this paper, we propose a distributed method that replaces clipping with smoothed normalization in the \algname{EF21} mechanism. 
Our method provides the first provable convergence in the DP setting and empirically outperforms the distributed version of \algname{DP-SGD} with smoothed normalization \citet{bu2024automatic, yang2022normalized}, a special case of \citet{das2021convergence}.

\paragraph{Error feedback.} 
Error feedback, or error compensation, has been applied to improve the convergence of distributed methods with gradient compression for communication-efficient learning.
First introduced by \citet{seide20141}, \algname{EF14} was extensively analyzed for first-order methods  in both single-node ~\citep{stich2018sparsified,karimireddy2019error,stich2019error,khirirat2019convergence} and distributed settings~\citep{wu2018error,alistarh2018convergence,gorbunov2020linearly,qian2021error,tang2019doublesqueeze,danilova2022distributed,qian2023catalyst}. 
Another error feedback variant is \algname{EF21}  proposed by \citet{richtarik2021ef21} that ensures strong convergence under any contractive compression operator for non-convex, smooth problems.  
Recent variants, e.g. \algname{EF21-SGD2M}~\citep{fatkhullin2024momentum}  and \algname{EControl}~\citep{gao2023econtrol},  have been developed to obtain the lower iteration and communication complexities than \algname{EF21} for stochastic optimization.

\section{Preliminaries}\label{sec:preliminaries}

\subsection{Notations}

We use $[a,b]$ to denote the set 
$\{a,a+1,a+2,\ldots,b\}$ for integers $a,b$ such that $a \leq b$, and $\Exp{u}$ to represent the expectation of a random variable $u$.
For vectors $x,y\in\R^d$, $\inp{x}{y}$ denotes their inner product, and $\norm{x} \eqdef \sqrt{\inp{x}{x}}$ denotes the Euclidean norm of $x$.
Finally, for functions $f,g:\R^d\rightarrow\R$, we write $f(x)=\cO(g(x))$ if $f(x) \leq M \cdot g(x)$ for some $M>0$.

\subsection{Problem Formulation}
We focus  on solving the 
finite-sum optimization problem:
\begin{align}\label{eqn:problem}
    \underset{x\in\R^d}{\min}  \  \left\{ f(x) \eqdef \frac{1}{n}\sum_{i=1}^n f_i(x) \right\},
\end{align}
where $x\in\R^d$ is the vector of model parameters of dimension $d$, and $f_i:\R^d\rightarrow\R$ is either a loss function on client $i \in [1,n]$ (distributed setting) or data point $i$ (single-node setting). 
Moreover, we impose the following assumption on objective functions commonly used for analyzing the convergence of first-order optimization algorithms~\citep{nesterov2018lectures}. 

\begin{assumption}\label{assum:lowebound_f}
Consider Problem~\eqref{eqn:problem}.
\begin{enumerate}
    \item  Let  $f:\R^d\rightarrow\R$ be bounded from below by a  finite constant $f^{\inf}$, i.e. $f(x) \geq f^{\inf} > -\infty$ for all $x\in\R^d$, and be $L$-smooth, i.e. $\norm{\nabla f(x) - \nabla f(y)} \leq L \norm{x-y}$ for all $x,y\in\R^d$. 
    \item Let  $f_i:\R^d\rightarrow\R$ be $L_i$-smooth, i.e. $\norm{\nabla f_i(x) - \nabla f_i(y)} \leq L \norm{x-y}$ for all $x,y\in\R^d$. 
\end{enumerate}
\end{assumption}

\subsection{DP-SGD}\label{sec:not_def}

To solve Problem~\eqref{eqn:problem}, the most common approach that ensures the approximate \((\epsilon, \delta)\)-differential privacy~\citep{dwork2006calibrating} is via the \algname{DP-SGD} method~\citep{abadi2016deep}
\begin{eqnarray}\label{eqn:DP_biased_GD}
    x^{k+1} = x^k - \gamma \left( \frac{1}{B}\sum_{i\in\cB^k} \Psi(\nabla f_i(x^k)) + z^k \right),
\end{eqnarray}
where $\gamma>0$ is the step size, $\cB^k$ is a subset of $\{1,2,\ldots,n\}$ with cardinality $\vert \cB^k \vert = B$, $z^k \in \R^d$ is the DP noise, and $\Psi:\R^d\rightarrow\R^d$ is an operator with bounded norm, i.e.  $\norm{\Psi(g)} \leq \Phi$ for some $\Phi>0$ and  any $g\in\R^d$.
The method~\eqref{eqn:DP_biased_GD} achieves $(\epsilon,\delta)$-DP~\citep{abadi2016deep} if  $z^k$  is zero-mean Gaussian noise with variance 
\begin{align}\label{eqn:variance_dp_sgd}
\sigma_{\text{DP}}^2 \geq  \Phi^2 \cdot    \frac{cB^2}{n^2}\frac{{K \log(1/\delta)}}{\epsilon^2},    
\end{align}
where $c>0$ is a constant, and $K>0$ is the total number of iterations. 
To obtain reasonable DP guarantees, we usually choose $\epsilon \leq 10$ and $\delta \ll 1/n$, where $n$ is the number of data points \citep{ponomareva2023dp}. Notice that the DP Gaussian noise variance \eqref{eqn:variance_dp_sgd} is scaled with the \textit{sensitivity} $\Phi$.

The method \eqref{eqn:DP_biased_GD} has been often analyzed, e.g. by \citet{zhang2020private,zhang2022understanding, murata2023diff2}, under the bounded gradient norm assumption
\begin{equation} \label{eq:bounded_grad}
    \norm{\nabla f_i(x)} \leq \Phi \quad \text{ for all \ $i$ \ and \ } x \in \R^d.
\end{equation}
However, this assumption has several limitations. 
Firstly, the sensitivity $\Phi$ is typically infeasible to compute for many loss functions used in training machine learning models. Even when it can be estimated, the resulting upper bound is often overly pessimistic, leading to excessively large DP noise and thus significantly degrading the algorithmic convergence performance.
Secondly, this assumption restricts the class of models and loss functions~$f$, as it excludes simple quadratic functions over unbounded domains.
Thirdly, this assumption is \say{pathological} in the distributed setting because it restricts the heterogeneity between different clients and can result in vacuous bounds \citep{khaled2020tighter}.

To enforce bounded sensitivity without imposing the bounded gradient norm, \citet{abadi2016deep} use 
clipping with threshold $\tau>0$
\begin{equation} \label{eq:clipping}
    \Clip{\tau}{g} \eqdef \min \br{1,\frac{\tau}{\norm{g}}} g. 
\end{equation}  
Here, the sensitivity $\Phi$ is the clipping threshold $\tau$, as 
$\norm{\Psi(g)} = \norm{\Clip{\tau}{g}} \leq \tau = \Phi$. 
In fact, the method \eqref{eqn:DP_biased_GD} that uses clipping \eqref{eq:clipping} is generally referred to as \algname{DP-SGD} in the literature. 
However, it is challenging to analyze the convergence of \algname{DP-SGD} without additional restrictive assumptions such as the symmetric noise assumption \citep{chen2020understanding, qian2021understanding}. 
Even without DP noise, \algname{DP-SGD} fails to converge due to the clipping bias~\citep{koloskova2023revisiting}.
Furthermore, since smaller values of $\tau$ imply stronger privacy but larger bias, jointly optimizing convergence and privacy of \algname{DP-SGD} by carefully tuning $\tau$ and $\gamma$ in the DP setting is a difficult task \citep{kurakin2022toward, bu2024automatic}.

\paragraph{Smoothed normalization as an alternative to clipping.}
To eliminate the need to tune the threshold $\tau$ of clipping, smoothed normalization is an operator alternative to clipping~\citep{bu2024automatic,yang2022normalized}, which is defined by  
\begin{align}\label{eqn:normalize}\Normalize{\alpha}{g} \eqdef \frac{1}{\alpha + \norm{g}}~g,
\end{align}
for some $\alpha  \geq 0$   
and satisfies the following property.
\begin{lemma}\label{lemma:prop_norm_vr2}
For any $\alpha  \geq 0$, $\beta>0$, and $g\in\R^d$, 
\begin{align}
\norm{ \Normalize{\alpha}{g} } & \leq 1, \label{eq:norm_1} \\ 
\sqnorm{g -   \beta  \Normalize{\alpha}{g}} &= \left( 1 - \frac{\beta}{\alpha+\norm{g}}\right)^2 \sqnorm{g}. \label{eq:norm_2}
\end{align}
\end{lemma}
Clearly, smoothed normalization ensures Property \eqref{eq:norm_1} that the norm of the normalized vector is bounded above by $1$. 
Also, Property \eqref{eq:norm_2} states that the distance between the true vector and a $\beta$-multiple of the normalized vector is bounded by a function of $\beta$, $\alpha$, and $\norm{g}$. 
Although smoothed normalization with $\alpha=0$ recovers  standard normalization ${g}/{\norm{g}}$~\citep{nesterov1984minimization,hazan2015beyond,levy2016power}, smoothed normalization with $\alpha>0$  improves the contraction factor, compared to standard normalization.
Specifically, as $\norm{g}\rightarrow 0$, the contraction factor of smoothed normalization approaches $(1-\beta/\alpha)^2$, while standard normalization lacks this contraction property.

Although \algname{DP-SGD} \eqref{eqn:DP_biased_GD} with smoothed normalization achieves robust empirical convergence in the DP setting~\citep{bu2024automatic},
its theoretical convergence is limited to the single-node setting and relies on restrictive assumptions, specifically the central symmetry of stochastic gradients around the true gradient.

\subsection{Limitations of DP Distributed Gradient Methods} 
Extending the convergence results of \algname{DP-SGD} to the distributed setting poses significant challenges due to {potential client heterogeneity}. 
Existing results often address the bias introduced by the operator (clipping or normalization) by relying on restrictive assumptions, such as imposing bounded gradient norms~\citep{li2022soteriafl, zhang2022understanding, murata2023diff2, wang2024efficient}, or assuming that clipping is effectively turned off~\citep{zhang2024private,noble2022differentially}. {Recently, \citet{li2024an} extended the analysis of \citet{koloskova2023revisiting} to a distributed private setting under strong gradient dissimilarity condition. However, their method fails to converge due to the clipping bias, as discussed in the previous section.} 
More importantly, even in the absence of the DP noise ($z^k=0$), the inherent bias in the gradient estimator can severely impact the convergence. 
For instance, \algname{DP-SGD}~\eqref{eqn:DP_biased_GD} diverge exponentially when $\Psi(\cdot)$ is a Top-$1$ compressor~\citep{beznosikov2023biased}, and fail to converge when  $\Psi(\cdot)$ is  clipping~\citep{chen2020understanding, khirirat2023clip21}. Similarly, smoothed normalization \eqref{eqn:normalize} with $\alpha=0$ also cannot address this convergence issue, as demonstrated in the next example.

\begin{example} \label{ex:counter}
Consider Problem~\eqref{eqn:problem} with $n=2, d=1$, $f_1 (x) = \frac{1}{2}\br{x-3}^2$ and $f_2 (x) = \frac{1}{2}\br{x+3}^2$. Then $f(x) = \frac{1}{2}(f_1(x) + f_2(x))$ satisfies Assumption \ref{assum:lowebound_f} and is minimized at $x^\star = 0$. The iterates $\{x^k\}$ generated by \eqref{eqn:DP_biased_GD} (for $B=2$) with $z^k=0$ and $\alpha=0$ do not progress when $x^0 = 2$, as the gradient estimator $\Normalize{\alpha}{\nabla f_1(x^k)} + \Normalize{\alpha}{\nabla f_2(x^k)}$ results in
\begin{equation*}
    \frac{\nabla f_1(x^0)}{\|\nabla f_1(x^0)\|} + \frac{\nabla f_2(x^0)}{\|\nabla f_2(x^0)\|} = -1/1 + 5/5 = 0.
\end{equation*}
\end{example}
Naively applying normalization to the gradients in \algname{DP-SGD} leads to the method that does not converge in the distributed setting without further assumptions.
Also, this fundamental limitation affects federated averaging algorithms that apply normalization on the client updates~\citep{das2021convergence}.

{
\renewcommand{\arraystretch}{1.5} 
\begin{table*}[]
    \centering
    \color{black}
\resizebox{0.9\textwidth}{!}{
    \begin{tabular}{|c|c|}
     \hline
      \bf Operator   &  \bf Property 
      \\ \hline 
      \begin{tabular}{c}
        Contractive compressor 
        { $\cC:\R^d\rightarrow\R^d$} 
      \end{tabular}
      &  
     $\sqnorm{\cC(g)-g} \leq (1-\eta)^2 \sqnorm{g}$ 
     \\ 
      \begin{tabular}{c}
        Clipping 
        { $\Clip{\tau}{g} \eqdef \min\left( 1, \frac{\tau}{\norm{g}}\right) g$}
      \end{tabular}
    &   $\sqnorm{\Clip{\tau}{g}-g} \leq \max( 0 , \norm{g} - \tau)^2$   
    \\
      \begin{tabular}{c}
        Smoothed normalization 
        { $\Normalize{\alpha}{g}\eqdef \frac{1}{\alpha + \norm{g}} g$}
      \end{tabular}
      &   $\sqnorm{\Normalize{\alpha}{g}-g} \leq \left(1- \frac{1}{\alpha+\norm{g}} \right)^2 \sqnorm{g}$  
      \\ \hline
    \end{tabular}}
    \caption{Smoothed normalization, unlike clipping, obtains the contractive property similar to compressors. 
    }
    \label{tab:clipping_normalization}
\end{table*}
}

\subsection{EF21 Mechanism} 

One approach to resolve the convergence issues of distributed gradient methods with biased operators is to use \algname{EF21}, an error feedback mechanism developed by~\citet{richtarik2021ef21}.
Instead of directly applying the biased gradient estimator $\Psi$ to the gradient, \algname{EF21} applies $\Psi$ to the \emph{difference} between the true gradient and the current error feedback vector.
At each iteration $k=0,1,\ldots,K$, each client $i$ receives the current iterate $x^k$ from the central server, and computes its local update $g_i^{k+1}$ via
\begin{eqnarray}\label{eqn:g_i_k_dist_biased_gradient}
    g_i^{k+1} = g_i^k + \beta \Psi(\nabla f_i(x^k)- g_i^k),
\end{eqnarray}
where $\beta > 0$. Next, the central server receives the average of local error feedback vectors that are communicated by all clients $\frac{1}{n}\sum_{i=1}^n\Psi(\nabla f_i(x^k)-g_i^k)$, computes the global gradient estimator $g^k \eqdef \frac{1}{n}\sum_{i=1}^n g_i^k$ as
\begin{eqnarray}\label{eqn:g_k_dist_biased_gradient}
    g^{k+1} = g^k + \frac{\beta}{n}\sum_{i=1}^n\Psi(\nabla f_i(x^k)-g_i^k),
\end{eqnarray}
and updates the next iterate $x^{k+1}$ via
\begin{eqnarray}\label{eqn:x_dist_biased_gradient}
    x^{k+1} = x^k - \gamma g^{k+1}.
\end{eqnarray}
This method generalizes \algname{EF21}, which utilizes a contractive compressor ~\citep{stich2018sparsified, beznosikov2023biased}  defined by  
$$  \sqnorm{g - \cC(g)} \leq (1-\eta)^2 \sqnorm{g},
$$ for some $\eta\in (0,1]$ and any $g\in\R^d$. For instance, the method encompasses other estimators $\Psi(\cdot)$ such as clipping in \algname{Clip21} proposed by \citet{khirirat2023clip21}.

Despite achieving the $\cO(1/K)$ convergence rate in the non-private setting, \algname{Clip21} faces difficulty in establishing the convergence in the presence of DP noise for two primary reasons. 
Firstly, its convergence analysis relies on separate descent inequalities when clipping turns on and off, as the operator does not satisfy the contractive compressor property required by \algname{EF21} (see Table~\ref{tab:clipping_normalization}). 
Secondly, the clipping threshold $\tau$ intricately influences both privacy and convergence. 
A sufficiently large $\tau$ is required to achieve the descent inequality, but this condition requires adding large Gaussian noise, which can prevent the convergence when it is accumulated. 
Due to these properties of clipping, analyzing the convergence of \algname{Clip21} in the DP setting is challenging.

\section{$\alpha$-NormEC in the Non-private Setting}\label{sec:norm21}

To address the convergence challenges of \algname{Clip21}, we propose \alg, the first distributed method to provide provable convergence guarantees in the DP setting. 
\alg implements the update rules defined by~\eqref{eqn:g_i_k_dist_biased_gradient},~\eqref{eqn:g_k_dist_biased_gradient}, and~\eqref{eqn:x_dist_biased_gradient}, where $\Psi(\cdot)$ is smoothed normalization~\eqref{eqn:normalize} that offers  key advantages over clipping. In the update rule in~\eqref{eqn:x_dist_biased_gradient}, we rather use server normalization $x^{k+1}=x^k-\gamma g^{k+1}/\norm{g^{k+1}}$ and adopt notation $0/0=0$. 
See Algorithm~\ref{alg:norm_norm21_dp_v1} for the detailed description of \alg.

\begin{algorithm}
\caption{\algname{(DP-)}\alg}
\label{alg:norm_norm21_dp_v1}
\begin{algorithmic}[1]
\STATE \textbf{Input:} {Step size} $\gamma>0$; $\beta>0$; normalization parameter $\alpha>0$;  starting points $x^0,{g_i^{0}} \in \R^d$ for $i\in [1,n]$ and $\hat g^{0} = \frac{1}{n}\sum_{i=1}^n g_i^{0}$; $z_i^k \in \R^d$ {are sampled from Gaussian distribution with zero mean and $\sigma_{\rm DP}^2$-variance.}
\FOR{each iteration $k = 0, 1, \dots, K$}
    \FOR{each client $i = 1, 2, \dots, n$ in parallel}
        \STATE Compute local  gradient $\nabla f_i(x^k)$
        \STATE Compute $\Delta_i^k = \Normalize{\alpha}{\nabla f_i(x^k) - g_i^{k}}$
        \STATE Update  {$g_i^{k+1} = g_i^{k} + \beta \Delta_i^k$}
        \STATE \textbf{Non-private setting:}  Transmit $\hat \Delta_i^k = \Delta_i^k$
        \STATE \textbf{Private setting:}  Transmit $\hat \Delta_i^k = \Delta_i^k + z_i^k$
    \ENDFOR
    \STATE Server computes $\hat g^{k+1} = \hat g^{k} + \frac{\beta}{n}\sum_{i=1}^n \hat \Delta_i^k$
    \STATE Server updates {$x^{k+1} = x^k - \gamma {\hat g^{k+1}}/{\norm{\hat g^{k+1}}}$} 
\ENDFOR
\STATE \textbf{Output:} $x^{K+1}$
\end{algorithmic}
\end{algorithm}

\alg achieves better convergence guarantees than \algname{Clip21} in the non-private setting and   the first convergence guarantees in the DP setting.
These theoretical benefits of \alg stem from favorable properties of smoothed normalization. 
Specifically, smoothed normalization, unlike clipping, behaves similarly to a contractive compressor (see Table~\ref{tab:clipping_normalization}), which simplifies the convergence analysis of \alg compared to \algname{Clip21}.

The first theorem presents the convergence of \alg in the non-private setting.

\begin{theorem}[Non-private setting]\label{thm:norm_ef21_nondp}
Consider  \alg (Algorithm~\ref{alg:norm_norm21_dp_v1}) for solving Problem~\eqref{eqn:problem}, where Assumption~\ref{assum:lowebound_f} holds. 
Let $\beta,\alpha,\gamma>0$ be chosen such that 
\begin{align*}
 \frac{\beta}{\alpha + R} < 1, \quad \text{and} \quad  \gamma \leq \frac{\beta R}{\alpha + R} \frac{1}{L_{\max}},  
\end{align*}
where $R = \max_{i\in[1,n]}\norm{\nabla f_i(x^0)-g_i^{0}}$ and $L_{\max}=\max_{i\in[1,n]} L_i$.  Then,  
\begin{eqnarray*}
    \min_{k\in[0,K]} \norm{\nabla f(x^k)}
    \leq  \frac{f(x^0)-f^{\inf}}{\gamma(K+1)} + 2R + \frac{L}{2}\gamma.
\end{eqnarray*}
\end{theorem}

From Theorem~\ref{thm:norm_ef21_nondp}, in the non-private setting,  \alg converges sublinearly up to the additive constant of $2R+ \frac{L}{2}\gamma$.
This constant diminishes when we properly choose initialized memory vectors $g_i^{-1}$ and the step size $\gamma$, as shown in the following corollary.

\begin{corollary}[Non-private setting]\label{corr:norm_ef21_nondp}
Consider  \alg (Algorithm~\ref{alg:norm_norm21_dp_v1})  for solving Problem~\eqref{eqn:problem} under the same setting as  Theorem~\ref{thm:norm_ef21_nondp}.  If we choose 
$g_i^{0}\in\R^d$  such that $\max_{i\in[1,n]} \norm{\nabla f_i(x^0)-g_i^{0}}  = D(K+1)^{-1/2}$ with any $D>0$, $\gamma \leq \frac{\beta}{L_{\max}} \frac{D}{\alpha+D} \frac{1}{(K+1)^{1/2}}$, and $\alpha>\beta$, then
\begin{eqnarray*}
     \min_{k\in[0,K]} \norm{\nabla f(x^k)}   \leq  \frac{C}{(K+1)^{1/2}}, 
\end{eqnarray*}
where $C =  \frac{L_{\max} (\alpha + D)}{\beta D} (f(x^0)-f^{\inf}) + 2D  + \frac{L}{2} \frac{\beta D}{L_{\max}(\alpha+D)}$.
\end{corollary}

From Corollary~\ref{corr:norm_ef21_nondp}, \alg enjoys the $\cO(1/\sqrt{K})$ convergence  in the gradient norm when we choose $g_i^{-1}$  such that $R  = \cO(1/\sqrt{K})$, and $\gamma = \cO(\beta/\sqrt{K})$. 
By further choosing $\alpha>1$, and 
\begin{eqnarray*}
    \beta = \frac{L_{\max}(\alpha+D)}{D}\sqrt{\frac{2(f(x^0)-f^{\inf})}{L}},
\end{eqnarray*}
which ensures $\frac{L_{\max} (\alpha + D)}{\beta D} (f(x^0)-f^{\inf}) = \frac{L}{2} \frac{\beta D}{L_{\max}(\alpha+D)}$, 
the associated convergence bound from Corollary~\ref{corr:norm_ef21_nondp} becomes 
\begin{eqnarray}\label{eqn:final_bound_norm21}
     \min_{k\in[0,K]} \norm{\nabla f(x^k)} 
     \leq  \frac{\sqrt{2L (f(x^0)-f^{\inf})} + 2D}{(K+1)^{1/2}}.
\end{eqnarray}
This bound comprises two terms. The first term $\sqrt{2L (f(x^0)-f^{\inf})}(K+1)^{-1/2}$ is the convergence bound obtained by classical gradient descent, while the second term $2D(K+1)^{-1/2}$ comes from the initialized memory vectors $g_i^{-1}$ for running the error-feedback mechanism.

\paragraph{Comparison between \alg and \algname{Clip21}.} In the non-private setting, \alg provides stronger convergence guarantees than \algname{Clip21}.
Firstly, the hyperparameters of \alg (\(\beta, \alpha, \gamma > 0\)), as defined in Theorem~\ref{thm:norm_ef21_nondp}, are easier to implement. The step size \(\gamma\) of \algname{Clip21} (Theorem 5.6 of \citet{khirirat2019convergence}) presents a practical challenge, as it depends on the inaccessible values of \(f(x^0) - f^{\inf}\). Furthermore, the convergence bound of \alg in \eqref{eqn:final_bound_norm21} exhibits a smaller convergence factor than that of \algname{Clip21}, as explained in Appendix~\ref{app:compare_norm21_clip21}. Specifically, by choosing $g_i^{0}\in\R^d$ such that $D$ is sufficiently small, the convergence bound of \alg in \eqref{eqn:final_bound_norm21} 
approaches that of classical gradient descent \citep{carmon2020lower}.

\paragraph{Proof outline.}
By the $L$-smoothness of the objective function $f$, and by the update for $x^{k+1}$ in \alg, we obtain
\begin{eqnarray*}
    V^{k+1} 
     \leq   V^k - \gamma \norm{\nabla f(x^k)}+ \frac{L\gamma^2}{2}  +  2\gamma W^k,
\end{eqnarray*}
where $V^k \eqdef f(x^k)-f^{\inf}$, and $W^k \eqdef \frac{1}{n}\sum_{i=1}^n \norm{\nabla f_i(x^k)-g_i^{k+1}}$. 
The key step to establish the convergence is to bound $\norm{\nabla f_i(x^k) - g_i^{k+1}}$. 
Using Lemma~\ref{lemma:norm21_approach2} and appropriate choices of the tuning parameters $\beta$, $\alpha$, and $\gamma$, we get
\begin{eqnarray*}
    \norm{\nabla f_i(x^k)-g_i^{k+1}} \leq \max_{i\in[1,n]}\norm{\nabla f_i(x^0)-g_i^{0}}, \quad \forall k \geq 0.
\end{eqnarray*}
Finally, substituting this bound into the previous descent inequality yields the convergence bound in $\min_{k \in [0, K]} \norm{\nabla f(x^k)}$.  
Deriving the bound on $\norm{\nabla f_i(x^k) - g_i^{k+1}}$ for \alg by induction is similar to but simpler than \algname{Clip21}. This simplified proof is possible because smoothed normalization possesses a contractive property similar to the contractive compressor used in \algname{EF21}.

\section{$\alpha$-NormEC in the DP Setting}\label{sec:dp_norm21}

Next, we extend \alg to the DP setting. Unlike \algname{Clip21}, \alg achieves the first provable convergence guarantees in the presence of DP noise, as the smoothed normalization parameter does not affect the DP noise variance.
\alg in the DP setting is almost identical to \alg in the non-private setting, except for the step of communicating $\hat \Delta_i^k$ of Algorithm~\ref{alg:norm_norm21_dp_v1}. 
In this step, rather than transmitting the normalized gradient $\hat{\Delta}_i^k = \Delta_i^k \eqdef \Normalize{\alpha}{\nabla f_i(x^k) - g_i^{k}}$ in the non-private setting, each client in the DP setting communicates the privatized normalized gradient $\hat{\Delta}_i^k = \Delta_i^k + z_i^k$, where $z_i^k$ is the Gaussian noise.

The next theorem presents the convergence of \alg in the DP setting.

\begin{theorem}[DP setting]\label{thm:dp_n_norm21}
Consider \algname{DP-}\alg (Algorithm~\ref{alg:norm_norm21_dp_v1})  for solving Problem~\eqref{eqn:problem}, where Assumption~\ref{assum:lowebound_f} holds. 
Let 
$\beta,\alpha,\gamma > 0$ be chosen such that 
\begin{eqnarray*}
    \frac{\beta}{\alpha + R } < 1, \quad \text{and} \quad \gamma \leq \frac{\beta R}{\alpha+R}\frac{1}{L_{\max}},
\end{eqnarray*}
where $R = \max_{i\in[1,n]}\norm{\nabla f_i(x^0) - g_i^{0}}$, and $L_{\max} = \max_{i\in[1,n]} L_i$. Then, 
\begin{eqnarray*}
    \min_{k\in[0,K]} \Exp{\norm{\nabla f(x^k)}} \leq  \frac{f(x^0)-f^{\inf}}{\gamma(K+1)} + 2R + \frac{L}{2}\gamma 
     + 2 \sqrt{\beta^2 (K+1)\sigma_{\rm DP}^2}.
\end{eqnarray*}
\end{theorem}

In the DP setting, from Theorem~\ref{thm:dp_n_norm21}, 
\alg achieves the sublinear convergence up to the additive term $2R+\frac{L}{2}\gamma+2\sqrt{\beta^2(K+1)\sigma_{\text{DP}}^2}$. Notice that \alg in the DP setting introduces one additional term that arises from the DP noise $\sigma_{\text{DP}}^2$.
This additive constant diminishes when we initialize memory vectors $g_i^{0}\in\R^d$ such that $R$ becomes small and properly choose parameters $\gamma,\beta>0$.

\paragraph{Utility guarantees. }
In the DP setting, unlike \algname{Clip21}~\citep{khirirat2023clip21}, \alg achieves  $(\epsilon,\delta)$-DP  and comes with convergence guarantees.
We show this by setting the standard deviation of the DP noise according to Theorem 1 of \citet{abadi2016deep}, i.e., $\sigma_{\rm DP}=\cO(\sqrt{(K+1)\log(1/\delta)}\epsilon^{-1})$, which yields the following utility bound.

\begin{corollary}[Utility guarantee in DP setting]\label{corr:dp_n_norm21}
Consider \algname{DP-}\alg (Algorithm~\ref{alg:norm_norm21_dp_v1})  for solving Problem~\eqref{eqn:problem} under the same setting as  Theorem~\ref{thm:dp_n_norm21}. If $\sigma_{\rm DP} = \cO(\sqrt{(K+1)\log(1/\delta)}\epsilon^{-1})$,  $\beta = \frac{\beta_0}{K+1}$ 
with $\beta_0 \leq  \Delta \sqrt[4]{n \epsilon^2/(d \log(1/\delta))}$, and $\alpha > \beta_0$, then Algorithm~\ref{alg:norm_norm21_dp_v1} satisfies $(\epsilon,\delta)$-DP and attains the bound:
\begin{eqnarray*}
  \min_{k\in[0,K]} \Exp{\norm{\nabla f(x^k)}} \leq    \cO\left(\Delta\sqrt[4]{\frac{d \log(1/\delta)}{n\epsilon^2}} \right)  + 2 R, 
\end{eqnarray*}
where $\Delta=\sqrt{{L_{\max}(\alpha +R)(f(x^0)-f^{\inf})}/{R}}$, and $R = {\max}_{i \in [1,n]}\norm{\nabla f_i(x^0)-g_i^{0}}$.
\end{corollary}

Unlike \algname{Clip21}, \alg provides the first utility bound in the DP distributed setting that accounts for the effect of smoothed normalization, a factor often neglected in the existing literature.
As $R$ is sufficiently small ($R\rightarrow 0$), \alg achieves the utility bound of  $\cO\left(\Delta\sqrt[4]{\frac{d \log(1/\delta)}{n\epsilon^2}} \right)$.
Our obtained utility bound applies for smooth problems without the bounded gradient norm assumption, the limitation present in prior works for analyzing \algname{DP-SGD} such as \citet{li2022soteriafl,pmlr-v216-wang23b,lowy2023private,zhang2020private}.

\section{Experiments}\label{sec:exp}

We evaluate the performance of \alg
to solve the non-convex optimization problem of deep neural network training in both non-private and private settings. 
We conducted experiments on the CIFAR-10 \citep{krizhevsky2009learning} dataset using the ResNet20 \citep{he2016deep} model for the image classification task. The compared methods are run for 300 communication rounds. The convergence plots present results for tuned step size $\gamma$.
Additional experimental details and results are provided in Appendix \ref{app:experiments}.

\subsection{Non-private Training}

\begin{figure}[h]
\begin{center}
\centerline{\includegraphics[width=0.9\linewidth]{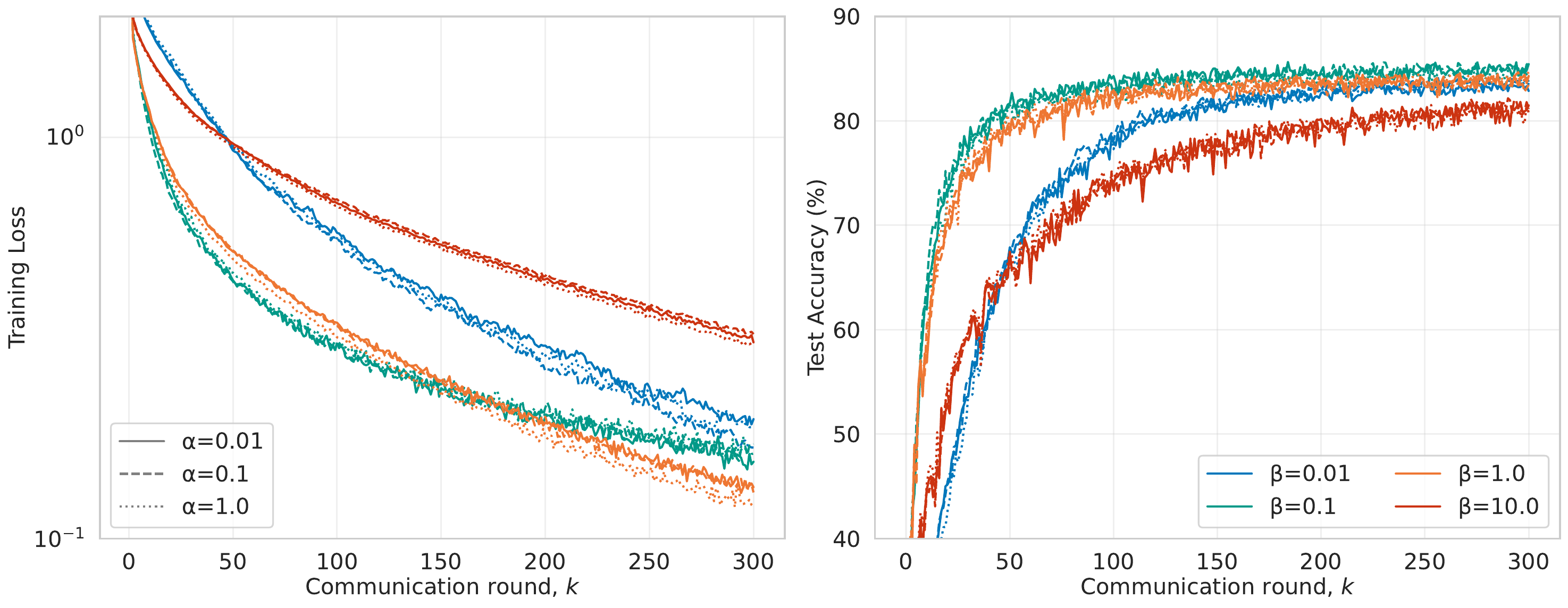}}
\caption{Training loss and test accuracy of non-private \alg with $\alpha=0.01$ [solid], $0.1$ [dashed], and $1.0$ [dotted], and $\beta=0.01$ [blue], $0.1$ [green], $1.0$ [orange], and $10.0$ [red].}
\label{fig:norm21_param_convergence}
\end{center}
\vspace{-35pt}
\end{figure}

\paragraph{\alg demonstrates stable convergence with respect to the normalization parameter $\alpha$, and robustness to variations in $\beta$ values.}
From Figure~\ref{fig:accuracy_heatmap} and~\ref{fig:norm21_param_convergence}, we observe that convergence of \alg is stable with respect to a wide range of $\alpha$ values and robust to variations in $\beta$.
The performance of \alg is primarily governed by the choice of $\beta$.
From Figure~\ref{fig:accuracy_heatmap},  optimal performance ($85$-$86\%$ accuracy) is observed when $\beta$ is around $0.1$.
While \alg is stable with respect to $\alpha$, extreme values of $\beta$ lead to suboptimal performance: very large values ($\beta=10.0$) result in lower accuracy ($81$-$82\%$), 
\begin{wrapfigure}{r}{0.5\textwidth}
\vspace{-5pt}
\centering
\includegraphics[width=\linewidth]{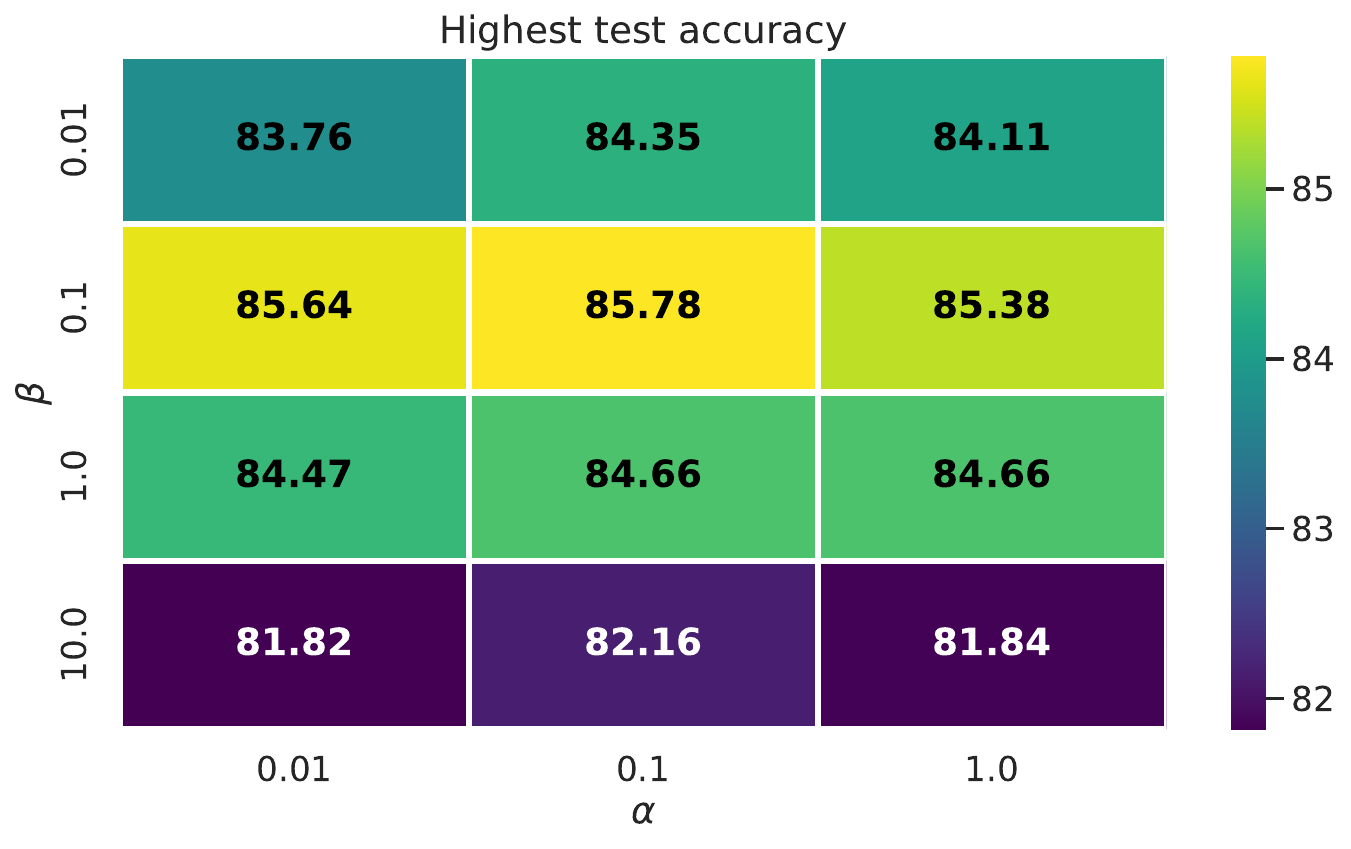}
\caption{The highest test accuracy by \alg with different $\alpha$ and $\beta$ values.}
\label{fig:accuracy_heatmap}
\vspace{-5pt}
\end{wrapfigure} 
while very small values ($\beta=0.01$) achieve moderate performance ($83$-$84\%$).
The optimal configuration, achieving the highest $85.78\%$ accuracy, is  $\beta=0.1$ and $\alpha=0.1$. 
Further analysis of the algorithm's stability with respect to $\alpha$ and robustness to $\beta$, including 
additional metrics is provided in Appendix \ref{app:exp_sensitivity}.
For subsequent experiments, we set $\alpha=0.01$, which is consistent with recommendations from prior work in the single-node setting \citep{bu2024automatic}.

\paragraph{Error compensation enables \alg to outperform \algname{DP-SGD}.}
From Figure \ref{fig:ef21_effect_loss}, \alg outperforms \algname{DP-SGD} with smoothed normalization (defined by Equation\eqref{eqn:DP_biased_GD} with $B = n$ and $z \equiv 0$).
This improvement is attributed to error compensation (EC), as confirmed by running \alg without server normalization (Line 11 of Algorithm~\ref{alg:norm_norm21_dp_v1}).
From Figure \ref{fig:ef21_effect_loss}, \alg achieves faster convergence and higher solution accuracy than \algname{DP-SGD} with smoothed normalization for most $\beta$ values, with the exception of $\beta=10$. 
However, such a large $\beta$ is impractical for differentially private training due to the resulting increase in noise variance. 
Moreover, while our algorithm demonstrates robust performance across varying $\beta$ values, \algname{DP-SGD} with smoothed normalization exhibits greater sensitivity to this parameter, notably struggling with convergence at $\beta = 0.01$. This comparison underscores how EC not only accelerates convergence but also improves the algorithm's stability with respect to parameter selection.
Further details and optimal parameters with corresponding final accuracies (Figure \ref{tab:norm_comparison}) are presented in Appendix \ref{app:exp_ef}.
\begin{figure}[h]
\begin{center}
\centerline{\includegraphics[width=0.9\linewidth]{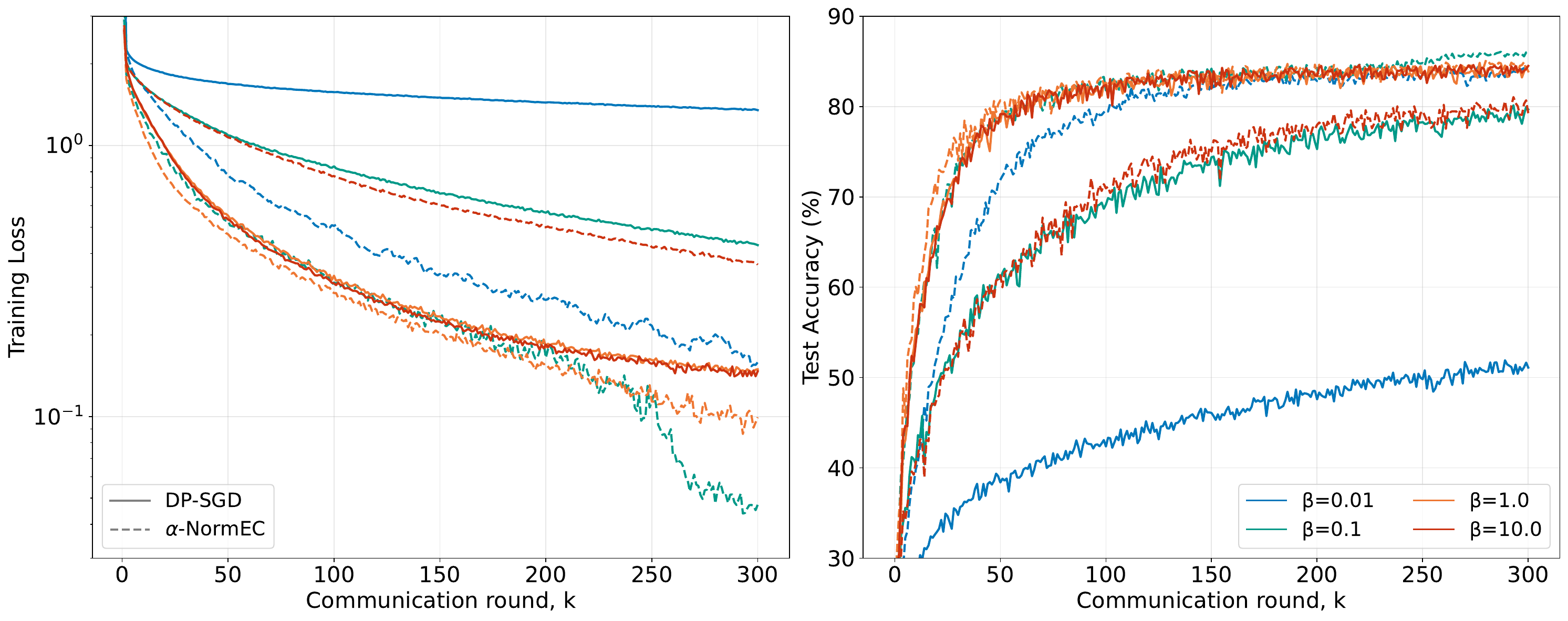}}
\caption{
Superior performance of \alg without server normalization  [dashed] over \algname{DP-SGD} \eqref{eqn:DP_biased_GD} with smoothed normalization [solid] in the non-private setting, in terms of  
training loss and test accuracy 
for different $\beta$ values (with fine-tuned step sizes).
}
\label{fig:ef21_effect_loss}
\end{center}
\vspace{-30pt}
\end{figure}

An ablation study examining the impact of server normalization on \alg is provided in Appendix \ref{app:exp_serv_norm}. Furthermore, a comparison between \alg and \algname{Clip21} is presented in Appendix \ref{app:exp_clip_norm}.

\subsection{Private Training}

We analyze the performance of \alg in the differentially private setting by training the model for 300 communication rounds. We set the noise variance at $\beta\sqrt{K \log(1/\delta)}\epsilon^{-1}$ for $\epsilon=8, \delta=10^{-5}$.

\begin{figure}[ht]
\begin{center}
\centerline{\includegraphics[width=\linewidth]{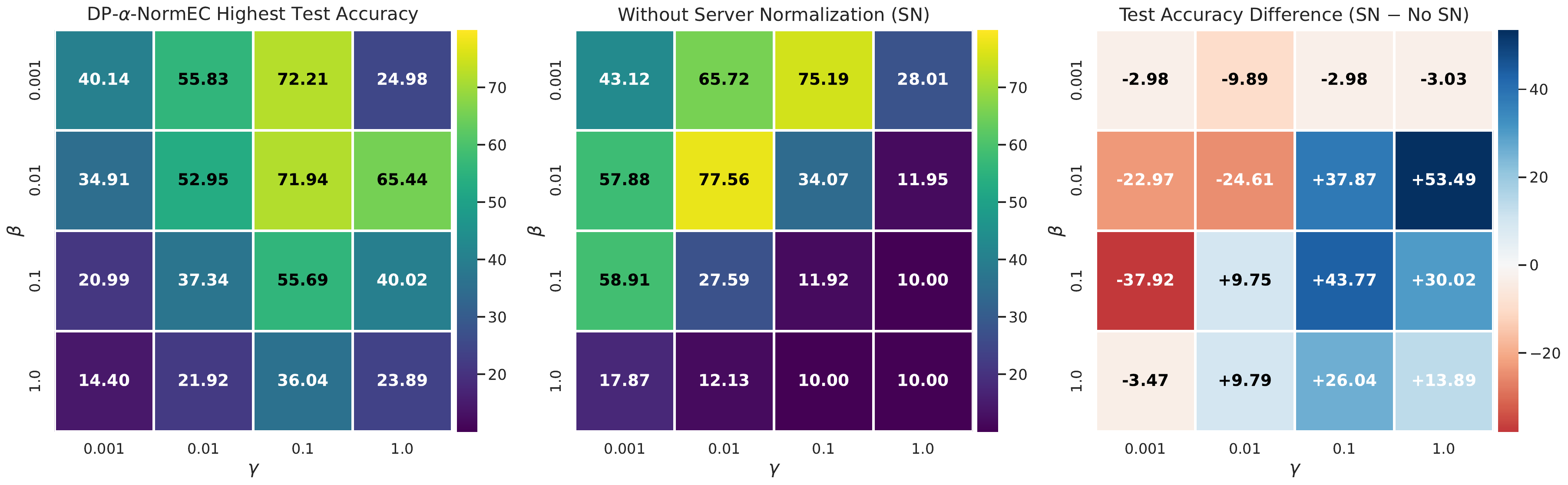}}
\caption{The highest test accuracy of \algname{DP}-\alg with [left] and without [center] Server Normalization (SN), and their difference [right].} 
\label{fig:dp_methods_accuracy}
\end{center}
\vspace{-35pt}
\end{figure}

\paragraph{Server normalization offers significantly improved convergence robustness.}
From Figure~\ref{fig:dp_methods_accuracy}, \alg with server normalization exhibits convergence robustness, albeit at the cost of slightly reduced peak performance, compared to \alg without server normalization. 
Specifically, at low $\beta$ values ($0.01$ and $0.001$), \alg with server normalization provides high robust performance in terms of the highest test accuracy.
For example, the performance variation across different $\gamma$ values ($0.001$, $0.01$, and $0.1$) is at most $6\%$ with server normalization, compared to a much larger $40\%$ variation without server normalization.
\begin{wrapfigure}{r}{0.4\textwidth}
\vspace{-18pt}
    \begin{center}
        \includegraphics[width=\linewidth]{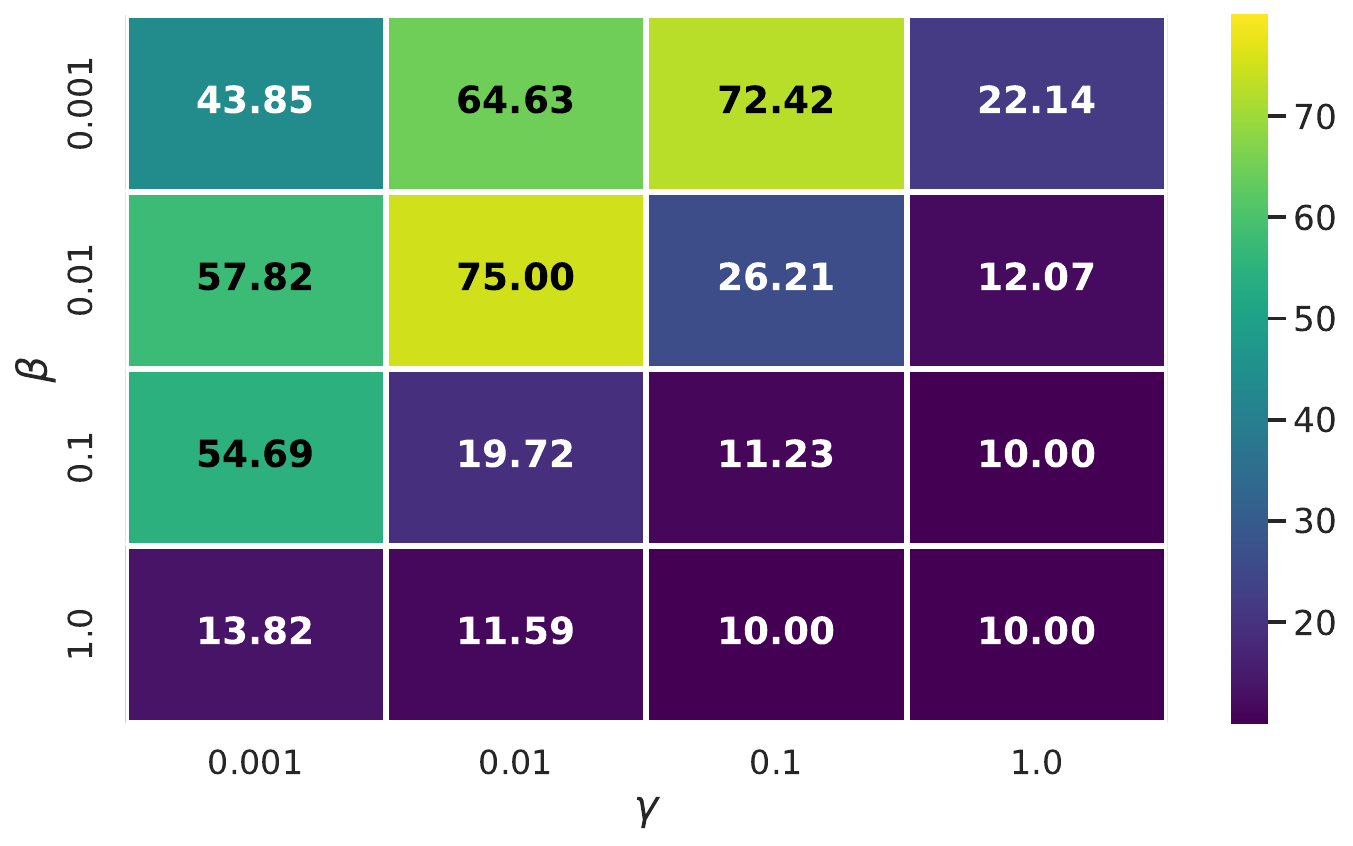}
    \end{center}
    \vspace{-15pt}
    \caption{The highest test accuracy of \algname{DP-Clip21}.}
    \label{fig:clip_accuracy}
    \vspace{-13pt}
\end{wrapfigure}
However, the highest test accuracy of $77.56\%$ is achieved by \alg without server normalization, outperforming \alg with server normalization, which reaches a peak accuracy of $72.21\%$.

\paragraph{\algname{DP-}\alg without server normalization outperforms \algname{DP-Clip21}.}
From Figure~\ref{fig:dp_methods}, under the same DP guarantee (the same $\beta$), \algname{DP-}\alg without server normalization converges significantly faster than \algname{DP-Clip21}\footnote{\color{black}\algname{DP-Clip21}, unlike \algname{Clip21}, does not have theoretical convergence guarantees.}, particularly when $\beta \in \{0.001,0.1\}$. 
Furthermore, comparing Figure~\ref{fig:dp_methods_accuracy}
 and~\ref{fig:clip_accuracy},  \algname{DP-}\alg provides higher robust convergence than \algname{DP-Clip21} across different hyperparameters $\beta, \gamma$. 
In particular, from Figure~\ref{fig:dp_methods}, while both algorithms diverge at $\beta=1.0$, only \algname{DP-}\alg with server normalization can still maintain convergence. 

\begin{figure}[H]
\begin{center}
\centerline{\includegraphics[width=0.9\linewidth]{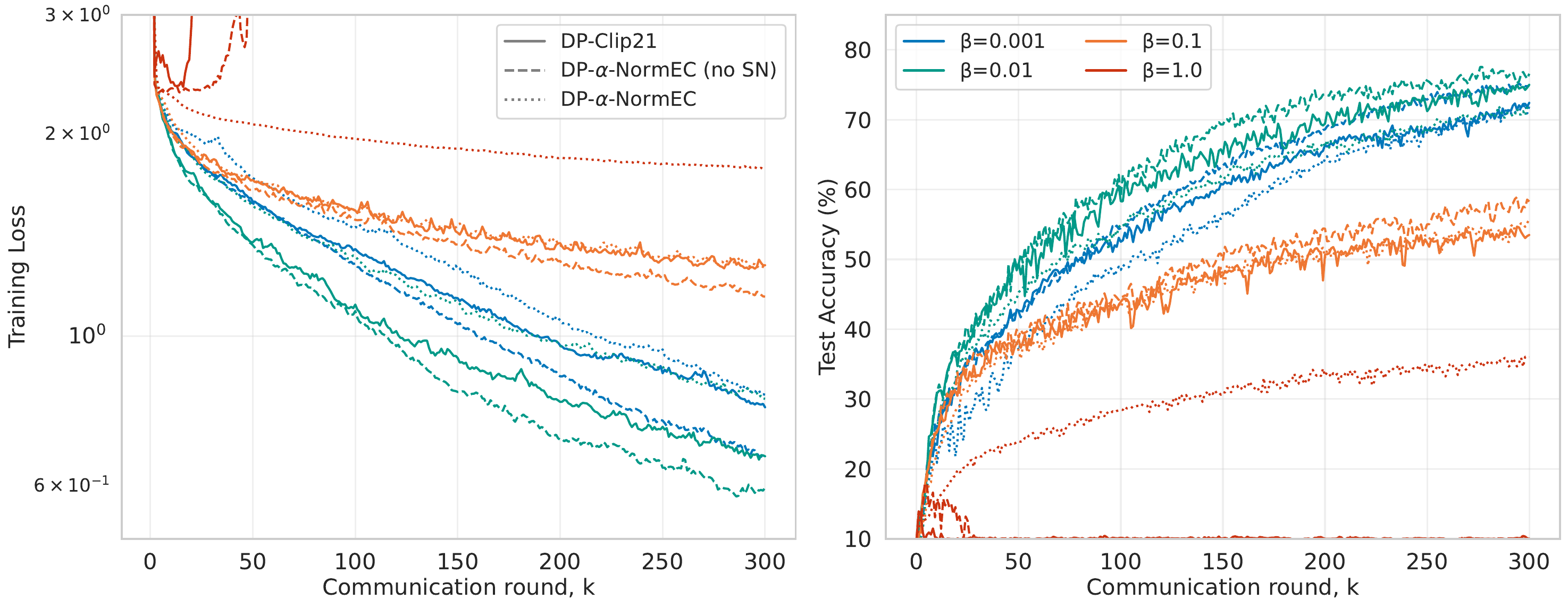}}
\caption{
Training loss and test accuracy of \algname{DP-Clip21} [solid], and \algname{DP-}\alg with [dotted] and without [dashed] server normalization (SN) across different $\beta$ values (with fine-tuned step sizes).}
\label{fig:dp_methods}
\end{center}
\end{figure}

\section{Conclusion}

We have proposed and analyzed \alg, a novel distributed algorithm that integrates smoothed normalization with the \algname{EF21} mechanism for solving non-convex, smooth optimization problems in both non-private and private settings. Unlike \algname{Clip21}, \alg achieves strong convergence guarantees that almost match those of classical gradient descent for non-private training and provides the first utility bound for private training without relying on restrictive assumptions such as bounded gradient norms.  
In neural network training, \alg achieves robust convergence across varying hyperparameters and significantly stronger convergence (due to error compensation) compared to \algname{DP-SGD} with smoothed normalization. In the private training, \algname{DP-}\alg benefits from server normalization for increased robustness and outperforms \algname{DP-Clip21}.

 Our work implies many promising research directions. 
 One direction is to extend \alg to accommodate the partial participation case, where the central server receives the local normalized gradients from a few clients, and the stochastic case, where each client has access only to stochastic gradients.  
 Another important direction is to modify \alg to solve federated learning problems, where the clients run their local updates before the local updates are normalized and transmitted to the central server.

\section*{Acknowledgements}
{\color{black}
We would like to thank Rustem Islamov for sharing his code.
} 

The research reported in this publication was supported by funding from King Abdullah University of Science and Technology (KAUST): i) KAUST Baseline Research Scheme, ii) Center of Excellence for Generative AI, under award number 5940, iii) SDAIA-KAUST Center of Excellence in Artificial Intelligence and Data Science.

\bibliographystyle{plainnat}
\bibliography{biblio}

\begin{thebibliography}{81}
\providecommand{\natexlab}[1]{#1}
\providecommand{\url}[1]{\texttt{#1}}
\expandafter\ifx\csname urlstyle\endcsname\relax
  \providecommand{\doi}[1]{doi: #1}\else
  \providecommand{\doi}{doi: \begingroup \urlstyle{rm}\Url}\fi

\bibitem[Abadi et~al.(2016)Abadi, Chu, Goodfellow, McMahan, Mironov, Talwar, and Zhang]{abadi2016deep}
Martin Abadi, Andy Chu, Ian Goodfellow, H~Brendan McMahan, Ilya Mironov, Kunal Talwar, and Li~Zhang.
\newblock Deep learning with differential privacy.
\newblock In \emph{Proceedings of the 2016 ACM SIGSAC conference on computer and communications security}, pages 308--318, 2016.

\bibitem[Alber et~al.(1998)Alber, Iusem, and Solodov]{alber1998projected}
Ya~I Alber, Alfredo~N Iusem, and Mikhail~V Solodov.
\newblock On the projected subgradient method for nonsmooth convex optimization in a hilbert space.
\newblock \emph{Mathematical Programming}, 81:\penalty0 23--35, 1998.

\bibitem[Alistarh et~al.(2018)Alistarh, Hoefler, Johansson, Konstantinov, Khirirat, and Renggli]{alistarh2018convergence}
Dan Alistarh, Torsten Hoefler, Mikael Johansson, Nikola Konstantinov, Sarit Khirirat, and C{\'e}dric Renggli.
\newblock The convergence of sparsified gradient methods.
\newblock \emph{Advances in Neural Information Processing Systems}, 31, 2018.

\bibitem[Andrew et~al.(2021)Andrew, Thakkar, McMahan, and Ramaswamy]{andrew2021differentially}
Galen Andrew, Om~Thakkar, H~Brendan McMahan, and Swaroop Ramaswamy.
\newblock Differentially private learning with adaptive clipping.
\newblock In \emph{Advances in Neural Information Processing Systems}, volume~34, pages 12191--12203, 2021.

\bibitem[Beznosikov et~al.(2023)Beznosikov, Horv{\'a}th, Richt{\'a}rik, and Safaryan]{beznosikov2023biased}
Aleksandr Beznosikov, Samuel Horv{\'a}th, Peter Richt{\'a}rik, and Mher Safaryan.
\newblock On biased compression for distributed learning.
\newblock \emph{Journal of Machine Learning Research}, 24\penalty0 (276):\penalty0 1--50, 2023.

\bibitem[Boenisch et~al.(2023)Boenisch, Dziedzic, Schuster, Shamsabadi, Shumailov, and Papernot]{boenisch2023curious}
Franziska Boenisch, Adam Dziedzic, Roei Schuster, Ali~Shahin Shamsabadi, Ilia Shumailov, and Nicolas Papernot.
\newblock When the curious abandon honesty: Federated learning is not private.
\newblock In \emph{2023 IEEE 8th European Symposium on Security and Privacy (EuroS{\&}P)}, pages 175--199. IEEE, 2023.

\bibitem[Brock et~al.(2021)Brock, De, Smith, and Simonyan]{brock2021high}
Andy Brock, Soham De, Samuel~L Smith, and Karen Simonyan.
\newblock High-performance large-scale image recognition without normalization.
\newblock In \emph{International conference on machine learning}, pages 1059--1071. PMLR, 2021.

\bibitem[Brown et~al.(2020)Brown, Mann, Ryder, Subbiah, Kaplan, Dhariwal, Neelakantan, Shyam, Sastry, Askell, et~al.]{brown2020language}
Tom Brown, Benjamin Mann, Nick Ryder, Melanie Subbiah, Jared~D Kaplan, Prafulla Dhariwal, Arvind Neelakantan, Pranav Shyam, Girish Sastry, Amanda Askell, et~al.
\newblock Language models are few-shot learners.
\newblock \emph{Advances in neural information processing systems}, 33:\penalty0 1877--1901, 2020.

\bibitem[Bu et~al.(2024)Bu, Wang, Zha, and Karypis]{bu2024automatic}
Zhiqi Bu, Yu-Xiang Wang, Sheng Zha, and George Karypis.
\newblock Automatic clipping: Differentially private deep learning made easier and stronger.
\newblock \emph{Advances in Neural Information Processing Systems}, 36, 2024.

\bibitem[Carmon et~al.(2020)Carmon, Duchi, Hinder, and Sidford]{carmon2020lower}
Yair Carmon, John~C Duchi, Oliver Hinder, and Aaron Sidford.
\newblock Lower bounds for finding stationary points i.
\newblock \emph{Mathematical Programming}, 184\penalty0 (1):\penalty0 71--120, 2020.

\bibitem[Chen et~al.(2020)Chen, Wu, and Hong]{chen2020understanding}
Xiangyi Chen, Steven~Z Wu, and Mingyi Hong.
\newblock Understanding gradient clipping in private sgd: A geometric perspective.
\newblock \emph{Advances in Neural Information Processing Systems}, 33:\penalty0 13773--13782, 2020.

\bibitem[Chezhegov et~al.(2024)Chezhegov, Klyukin, Semenov, Beznosikov, Gasnikov, Horv{\'a}th, Tak{\'a}{\v{c}}, and Gorbunov]{chezhegov2024gradient}
Savelii Chezhegov, Yaroslav Klyukin, Andrei Semenov, Aleksandr Beznosikov, Alexander Gasnikov, Samuel Horv{\'a}th, Martin Tak{\'a}{\v{c}}, and Eduard Gorbunov.
\newblock Gradient clipping improves {AdaGrad} when the noise is heavy-tailed.
\newblock \emph{arXiv preprint arXiv:2406.04443}, 2024.

\bibitem[Crawshaw et~al.(2023)Crawshaw, Bao, and Liu]{crawshaw2023episode}
Michael Crawshaw, Yajie Bao, and Mingrui Liu.
\newblock {EPISODE}: Episodic gradient clipping with periodic resampled corrections for federated learning with heterogeneous data.
\newblock In \emph{The Eleventh International Conference on Learning Representations}, 2023.

\bibitem[Danilova and Gorbunov(2022)]{danilova2022distributed}
Marina Danilova and Eduard Gorbunov.
\newblock Distributed methods with absolute compression and error compensation.
\newblock In \emph{International Conference on Mathematical Optimization Theory and Operations Research}, pages 163--177. Springer, 2022.

\bibitem[Das et~al.(2022)Das, Hashemi, sujay sanghavi, and Dhillon]{das2021convergence}
Rudrajit Das, Abolfazl Hashemi, sujay sanghavi, and Inderjit~S Dhillon.
\newblock Differentially private federated learning with normalized updates.
\newblock In \emph{OPT 2022: Optimization for Machine Learning (NeurIPS 2022 Workshop)}, 2022.

\bibitem[Dwork et~al.(2006)Dwork, McSherry, Nissim, and Smith]{dwork2006calibrating}
Cynthia Dwork, Frank McSherry, Kobbi Nissim, and Adam Smith.
\newblock Calibrating noise to sensitivity in private data analysis.
\newblock In \emph{Theory of Cryptography: Third Theory of Cryptography Conference, TCC 2006, New York, NY, USA, March 4-7, 2006. Proceedings 3}, pages 265--284. Springer, 2006.

\bibitem[Dwork et~al.(2014)Dwork, Roth, et~al.]{dwork2014algorithmic}
Cynthia Dwork, Aaron Roth, et~al.
\newblock The algorithmic foundations of differential privacy.
\newblock \emph{Foundations and Trends{\textregistered} in Theoretical Computer Science}, 9\penalty0 (3--4):\penalty0 211--407, 2014.

\bibitem[Ermoliev(1988)]{ermoliev1988stochastic}
Yuri Ermoliev.
\newblock Stochastic quasigradient methods. numerical techniques for stochastic optimization.
\newblock \emph{Springer Series in Computational Mathematics}, \penalty0 (10):\penalty0 141--185, 1988.

\bibitem[Fatkhullin et~al.(2024)Fatkhullin, Tyurin, and Richt{\'a}rik]{fatkhullin2024momentum}
Ilyas Fatkhullin, Alexander Tyurin, and Peter Richt{\'a}rik.
\newblock Momentum provably improves error feedback!
\newblock \emph{Advances in Neural Information Processing Systems}, 36, 2024.

\bibitem[Gao et~al.(2024)Gao, Islamov, and Stich]{gao2023econtrol}
Yuan Gao, Rustem Islamov, and Sebastian~U Stich.
\newblock {EC}ontrol: Fast distributed optimization with compression and error control.
\newblock In \emph{The Twelfth International Conference on Learning Representations}, 2024.

\bibitem[Gorbunov et~al.(2020{\natexlab{a}})Gorbunov, Danilova, and Gasnikov]{gorbunov2020stochastic}
Eduard Gorbunov, Marina Danilova, and Alexander Gasnikov.
\newblock Stochastic optimization with heavy-tailed noise via accelerated gradient clipping.
\newblock \emph{Advances in Neural Information Processing Systems}, 33:\penalty0 15042--15053, 2020{\natexlab{a}}.

\bibitem[Gorbunov et~al.(2020{\natexlab{b}})Gorbunov, Kovalev, Makarenko, and Richt{\'a}rik]{gorbunov2020linearly}
Eduard Gorbunov, Dmitry Kovalev, Dmitry Makarenko, and Peter Richt{\'a}rik.
\newblock Linearly converging error compensated {SGD}.
\newblock \emph{Advances in Neural Information Processing Systems}, 33:\penalty0 20889--20900, 2020{\natexlab{b}}.

\bibitem[Gorbunov et~al.(2024)Gorbunov, Sadiev, Danilova, Horv{\'a}th, Gidel, Dvurechensky, Gasnikov, and Richt{\'a}rik]{gorbunov2024highprobability}
Eduard Gorbunov, Abdurakhmon Sadiev, Marina Danilova, Samuel Horv{\'a}th, Gauthier Gidel, Pavel Dvurechensky, Alexander Gasnikov, and Peter Richt{\'a}rik.
\newblock High-probability convergence for composite and distributed stochastic minimization and variational inequalities with heavy-tailed noise.
\newblock In \emph{Forty-first International Conference on Machine Learning}, 2024.

\bibitem[Gorbunov et~al.(2025)Gorbunov, Tupitsa, Choudhury, Aliev, Richt{\'a}rik, Horv{\'a}th, and Tak{\'a}{\v{c}}]{gorbunov2024methods}
Eduard Gorbunov, Nazarii Tupitsa, Sayantan Choudhury, Alen Aliev, Peter Richt{\'a}rik, Samuel Horv{\'a}th, and Martin Tak{\'a}{\v{c}}.
\newblock Methods for convex {$(L_0,L_1)$}-smooth optimization: Clipping, acceleration, and adaptivity.
\newblock In \emph{The Thirteenth International Conference on Learning Representations}, 2025.

\bibitem[Hazan et~al.(2015)Hazan, Levy, and Shalev-Shwartz]{hazan2015beyond}
Elad Hazan, Kfir Levy, and Shai Shalev-Shwartz.
\newblock Beyond convexity: Stochastic quasi-convex optimization.
\newblock \emph{Advances in neural information processing systems}, 28, 2015.

\bibitem[He et~al.(2016)He, Zhang, Ren, and Sun]{he2016deep}
Kaiming He, Xiangyu Zhang, Shaoqing Ren, and Jian Sun.
\newblock Deep residual learning for image recognition.
\newblock In \emph{Proceedings of the IEEE Conference on Computer Vision and Pattern Recognition (CVPR)}, pages 770--778, 2016.

\bibitem[H{\"u}bler et~al.(2024)H{\"u}bler, Yang, Li, and He]{hubler2024parameter}
Florian H{\"u}bler, Junchi Yang, Xiang Li, and Niao He.
\newblock Parameter-agnostic optimization under relaxed smoothness.
\newblock In \emph{International Conference on Artificial Intelligence and Statistics}, pages 4861--4869. PMLR, 2024.

\bibitem[H{\"u}bler et~al.(2025)H{\"u}bler, Fatkhullin, and He]{hubler2024gradient}
Florian H{\"u}bler, Ilyas Fatkhullin, and Niao He.
\newblock From gradient clipping to normalization for heavy tailed {SGD}.
\newblock In \emph{The 28th International Conference on Artificial Intelligence and Statistics}, 2025.

\bibitem[Idelbayev()]{idelbayev18a}
Yerlan Idelbayev.
\newblock Proper {ResNet} implementation for {CIFAR10/CIFAR100} in {PyTorch}.
\newblock \url{https://github.com/akamaster/pytorch_resnet_cifar10}.
\newblock Accessed: 2024-12-31.

\bibitem[Kairouz et~al.(2021)Kairouz, McMahan, Avent, Bellet, Bennis, Bhagoji, Bonawitz, Charles, Cormode, Cummings, D{'}Oliveira, Eichner, Rouayheb, Evans, Gardner, Garrett, Gasc{\'{o}}n, Ghazi, Gibbons, Gruteser, Harchaoui, He, He, Huo, Hutchinson, Hsu, Jaggi, Javidi, Joshi, Khodak, Kone{\v{c}}n{\'y}, Korolova, Koushanfar, Koyejo, Lepoint, Liu, Mittal, Mohri, Nock, {\"{O}}zg{\"{u}}r, Pagh, Qi, Ramage, Raskar, Raykova, Song, Song, Stich, Sun, Suresh, Tram{\`{e}}r, Vepakomma, Wang, Xiong, Xu, Yang, Yu, Yu, and Zhao]{kairouz2021advances}
Peter Kairouz, H.~Brendan McMahan, Brendan Avent, Aur{\'{e}}lien Bellet, Mehdi Bennis, Arjun~Nitin Bhagoji, Kallista~A. Bonawitz, Zachary Charles, Graham Cormode, Rachel Cummings, Rafael G.~L. D{'}Oliveira, Hubert Eichner, Salim~El Rouayheb, David Evans, Josh Gardner, Zachary Garrett, Adri{\`{a}} Gasc{\'{o}}n, Badih Ghazi, Phillip~B. Gibbons, Marco Gruteser, Zaid Harchaoui, Chaoyang He, Lie He, Zhouyuan Huo, Ben Hutchinson, Justin Hsu, Martin Jaggi, Tara Javidi, Gauri Joshi, Mikhail Khodak, Jakub Kone{\v{c}}n{\'y}, Aleksandra Korolova, Farinaz Koushanfar, Sanmi Koyejo, Tancr{\`{e}}de Lepoint, Yang Liu, Prateek Mittal, Mehryar Mohri, Richard Nock, Ayfer {\"{O}}zg{\"{u}}r, Rasmus Pagh, Hang Qi, Daniel Ramage, Ramesh Raskar, Mariana Raykova, Dawn Song, Weikang Song, Sebastian~U. Stich, Ziteng Sun, Ananda~Theertha Suresh, Florian Tram{\`{e}}r, Praneeth Vepakomma, Jianyu Wang, Li~Xiong, Zheng Xu, Qiang Yang, Felix~X. Yu, Han Yu, and Sen Zhao.
\newblock Advances and open problems in federated learning.
\newblock \emph{Found. Trends Mach. Learn.}, 14\penalty0 (1-2):\penalty0 1--210, 2021.
\newblock \doi{10.1561/2200000083}.
\newblock URL \url{https://doi.org/10.1561/2200000083}.

\bibitem[Karimireddy et~al.(2019)Karimireddy, Rebjock, Stich, and Jaggi]{karimireddy2019error}
Sai~Praneeth Karimireddy, Quentin Rebjock, Sebastian Stich, and Martin Jaggi.
\newblock Error feedback fixes sign{SGD} and other gradient compression schemes.
\newblock In \emph{International Conference on Machine Learning}, pages 3252--3261. PMLR, 2019.

\bibitem[Karimireddy et~al.(2021)Karimireddy, He, and Jaggi]{karimireddy2021learning}
Sai~Praneeth Karimireddy, Lie He, and Martin Jaggi.
\newblock Learning from history for byzantine robust optimization.
\newblock In \emph{International Conference on Machine Learning}, pages 5311--5319. PMLR, 2021.

\bibitem[Khaled et~al.(2020)Khaled, Mishchenko, and Richt{\'a}rik]{khaled2020tighter}
Ahmed Khaled, Konstantin Mishchenko, and Peter Richt{\'a}rik.
\newblock Tighter theory for local {SGD} on identical and heterogeneous data.
\newblock In \emph{International Conference on Artificial Intelligence and Statistics}, pages 4519--4529. PMLR, 2020.

\bibitem[Khirirat et~al.(2019)Khirirat, Magn{\'u}sson, and Johansson]{khirirat2019convergence}
Sarit Khirirat, Sindri Magn{\'u}sson, and Mikael Johansson.
\newblock Convergence bounds for compressed gradient methods with memory based error compensation.
\newblock In \emph{ICASSP 2019-2019 IEEE International Conference on Acoustics, Speech and Signal Processing (ICASSP)}, pages 2857--2861. IEEE, 2019.

\bibitem[Khirirat et~al.(2023)Khirirat, Gorbunov, Horv{\'a}th, Islamov, Karray, and Richt{\'a}rik]{khirirat2023clip21}
Sarit Khirirat, Eduard Gorbunov, Samuel Horv{\'a}th, Rustem Islamov, Fakhri Karray, and Peter Richt{\'a}rik.
\newblock Clip21: Error feedback for gradient clipping.
\newblock \emph{arXiv preprint arXiv:2305.18929}, 2023.

\bibitem[Koloskova et~al.(2023)Koloskova, Hendrikx, and Stich]{koloskova2023revisiting}
Anastasia Koloskova, Hadrien Hendrikx, and Sebastian~U Stich.
\newblock Revisiting gradient clipping: Stochastic bias and tight convergence guarantees.
\newblock In \emph{International Conference on Machine Learning}, pages 17343--17363. PMLR, 2023.

\bibitem[Kone{\v{c}}n{\'y} et~al.(2016)Kone{\v{c}}n{\'y}, McMahan, Yu, Richt{\'a}rik, Suresh, and Bacon]{konecny2017federated}
Jakub Kone{\v{c}}n{\'y}, H.~Brendan McMahan, Felix~X. Yu, Peter Richt{\'a}rik, Ananda~Theertha Suresh, and Dave Bacon.
\newblock Federated learning: Strategies for improving communication efficiency.
\newblock \emph{NIPS Private Multi-Party Machine Learning Workshop}, 2016.

\bibitem[Krizhevsky et~al.(2009)Krizhevsky, Hinton, et~al.]{krizhevsky2009learning}
Alex Krizhevsky, Geoffrey Hinton, et~al.
\newblock Learning multiple layers of features from tiny images.
\newblock Technical report, University of Toronto, Toronto, 2009.

\bibitem[Kurakin et~al.(2022)Kurakin, Song, Chien, Geambasu, Terzis, and Thakurta]{kurakin2022toward}
Alexey Kurakin, Shuang Song, Steve Chien, Roxana Geambasu, Andreas Terzis, and Abhradeep Thakurta.
\newblock Toward training at imagenet scale with differential privacy.
\newblock \emph{arXiv preprint arXiv:2201.12328}, 2022.

\bibitem[Levy(2016)]{levy2016power}
Kfir~Y Levy.
\newblock The power of normalization: Faster evasion of saddle points.
\newblock \emph{arXiv preprint arXiv:1611.04831}, 2016.

\bibitem[Li et~al.(2024)Li, Jiang, Schmidt, Alstr{\o}m, and Stich]{li2024an}
Bo~Li, Xiaowen Jiang, Mikkel~N. Schmidt, Tommy~Sonne Alstr{\o}m, and Sebastian~U Stich.
\newblock An improved analysis of per-sample and per-update clipping in federated learning.
\newblock In \emph{The Twelfth International Conference on Learning Representations}, 2024.

\bibitem[Li et~al.(2022)Li, Zhao, Li, and Chi]{li2022soteriafl}
Zhize Li, Haoyu Zhao, Boyue Li, and Yuejie Chi.
\newblock {SoteriaFL}: A unified framework for private federated learning with communication compression.
\newblock \emph{Advances in Neural Information Processing Systems}, 35:\penalty0 4285--4300, 2022.

\bibitem[Liu et~al.(2022)Liu, Zhuang, Lei, and Liao]{liu2022communication}
Mingrui Liu, Zhenxun Zhuang, Yunwen Lei, and Chunyang Liao.
\newblock A communication-efficient distributed gradient clipping algorithm for training deep neural networks.
\newblock \emph{Advances in Neural Information Processing Systems}, 35:\penalty0 26204--26217, 2022.

\bibitem[Lobanov et~al.(2024)Lobanov, Gasnikov, Gorbunov, and Tak{\'a}c]{lobanov2024linear}
Aleksandr Lobanov, Alexander Gasnikov, Eduard Gorbunov, and Martin Tak{\'a}c.
\newblock Linear convergence rate in convex setup is possible! gradient descent method variants under {$(L_0, L_1)$}-smoothness.
\newblock \emph{arXiv preprint arXiv:2412.17050}, 2024.

\bibitem[Lowy et~al.(2023)Lowy, Ghafelebashi, and Razaviyayn]{lowy2023private}
Andrew Lowy, Ali Ghafelebashi, and Meisam Razaviyayn.
\newblock Private non-convex federated learning without a trusted server.
\newblock In \emph{International Conference on Artificial Intelligence and Statistics}, pages 5749--5786. PMLR, 2023.

\bibitem[Malinovsky et~al.(2023)Malinovsky, Gorbunov, Horv{\'a}th, and Richt{\'a}rik]{malinovsky2023byzantine}
Grigory Malinovsky, Eduard Gorbunov, Samuel Horv{\'a}th, and Peter Richt{\'a}rik.
\newblock Byzantine robustness and partial participation can be achieved simultaneously: Just clip gradient differences.
\newblock In \emph{Privacy Regulation and Protection in Machine Learning}, 2023.

\bibitem[McMahan et~al.(2017)McMahan, Moore, Ramage, Hampson, and y~Arcas]{mcmahan2017communication}
Brendan McMahan, Eider Moore, Daniel Ramage, Seth Hampson, and Blaise~Aguera y~Arcas.
\newblock Communication-efficient learning of deep networks from decentralized data.
\newblock In \emph{Artificial Intelligence and Statistics}, pages 1273--1282. PMLR, 2017.

\bibitem[McMahan et~al.(2018)McMahan, Ramage, Talwar, and Zhang]{mcmahan2018learning}
H~Brendan McMahan, Daniel Ramage, Kunal Talwar, and Li~Zhang.
\newblock Learning differentially private recurrent language models.
\newblock In \emph{International Conference on Learning Representations}, 2018.

\bibitem[Merad and Ga{\"\i}ffas(2024)]{merad2023robust}
Ibrahim Merad and St{\'e}phane Ga{\"\i}ffas.
\newblock Robust stochastic optimization via gradient quantile clipping.
\newblock \emph{Transactions on Machine Learning Research}, 2024.

\bibitem[Merity et~al.(2018)Merity, Keskar, and Socher]{merity2017regularizing}
Stephen Merity, Nitish~Shirish Keskar, and Richard Socher.
\newblock Regularizing and optimizing {LSTM} language models.
\newblock In \emph{International Conference on Learning Representations}, 2018.

\bibitem[Murata and Suzuki(2023)]{murata2023diff2}
Tomoya Murata and Taiji Suzuki.
\newblock Diff2: Differential private optimization via gradient differences for nonconvex distributed learning.
\newblock In \emph{International Conference on Machine Learning}, pages 25523--25548. PMLR, 2023.

\bibitem[Nesterov et~al.(2018)]{nesterov2018lectures}
Yurii Nesterov et~al.
\newblock \emph{Lectures on convex optimization}, volume 137.
\newblock Springer, 2018.

\bibitem[Nesterov(1984)]{nesterov1984minimization}
Yurii~E Nesterov.
\newblock Minimization methods for nonsmooth convex and quasiconvex functions.
\newblock \emph{Matekon}, 29\penalty0 (3):\penalty0 519--531, 1984.

\bibitem[Nguyen et~al.(2023)Nguyen, Nguyen, Ene, and Nguyen]{nguyen2023improved}
Ta~Duy Nguyen, Thien~H Nguyen, Alina Ene, and Huy Nguyen.
\newblock Improved convergence in high probability of clipped gradient methods with heavy tailed noise.
\newblock \emph{Advances in Neural Information Processing Systems}, 36:\penalty0 24191--24222, 2023.

\bibitem[Noble et~al.(2022)Noble, Bellet, and Dieuleveut]{noble2022differentially}
Maxence Noble, Aur{\'e}lien Bellet, and Aymeric Dieuleveut.
\newblock Differentially private federated learning on heterogeneous data.
\newblock In \emph{International Conference on Artificial Intelligence and Statistics}, pages 10110--10145. PMLR, 2022.

\bibitem[{\"O}zfatura et~al.(2023){\"O}zfatura, {\"O}zfatura, K{\"u}p{\c{c}}{\"u}, and Gunduz]{ozfatura2023byzantines}
Kerem {\"O}zfatura, Emre {\"O}zfatura, Alptekin K{\"u}p{\c{c}}{\"u}, and Deniz Gunduz.
\newblock Byzantines can also learn from history: Fall of centered clipping in federated learning.
\newblock \emph{IEEE Transactions on Information Forensics and Security}, 19:\penalty0 2010--2022, 2023.

\bibitem[Papernot and Steinke(2022)]{papernot2021hyperparameter}
Nicolas Papernot and Thomas Steinke.
\newblock Hyperparameter tuning with renyi differential privacy.
\newblock In \emph{International Conference on Learning Representations}, 2022.

\bibitem[Pascanu et~al.(2013)Pascanu, Mikolov, and Bengio]{pascanu2013difficulty}
Razvan Pascanu, Tomas Mikolov, and Yoshua Bengio.
\newblock On the difficulty of training recurrent neural networks.
\newblock In \emph{Proceedings of the 30th International Conference on Machine Learning}, pages 1310--1318. PMLR, 2013.

\bibitem[Ponomareva et~al.(2023)Ponomareva, Hazimeh, Kurakin, Xu, Denison, McMahan, Vassilvitskii, Chien, and Thakurta]{ponomareva2023dp}
Natalia Ponomareva, Hussein Hazimeh, Alex Kurakin, Zheng Xu, Carson Denison, H~Brendan McMahan, Sergei Vassilvitskii, Steve Chien, and Abhradeep~Guha Thakurta.
\newblock How to dp-fy ml: A practical guide to machine learning with differential privacy.
\newblock \emph{Journal of Artificial Intelligence Research}, 77:\penalty0 1113--1201, 2023.

\bibitem[Qian et~al.(2021{\natexlab{a}})Qian, Wu, Zhuang, Wang, and Xiao]{qian2021understanding}
Jiang Qian, Yuren Wu, Bojin Zhuang, Shaojun Wang, and Jing Xiao.
\newblock Understanding gradient clipping in incremental gradient methods.
\newblock In \emph{International Conference on Artificial Intelligence and Statistics}, pages 1504--1512. PMLR, 2021{\natexlab{a}}.

\bibitem[Qian et~al.(2021{\natexlab{b}})Qian, Richt{\'a}rik, and Zhang]{qian2021error}
Xun Qian, Peter Richt{\'a}rik, and Tong Zhang.
\newblock Error compensated distributed {SGD} can be accelerated.
\newblock \emph{Advances in Neural Information Processing Systems}, 34:\penalty0 30401--30413, 2021{\natexlab{b}}.

\bibitem[Qian et~al.(2023)Qian, Dong, Zhang, and Richtarik]{qian2023catalyst}
Xun Qian, Hanze Dong, Tong Zhang, and Peter Richtarik.
\newblock Catalyst acceleration of error compensated methods leads to better communication complexity.
\newblock In \emph{International Conference on Artificial Intelligence and Statistics}, pages 615--649. PMLR, 2023.

\bibitem[Richt{\'a}rik et~al.(2021)Richt{\'a}rik, Sokolov, and Fatkhullin]{richtarik2021ef21}
Peter Richt{\'a}rik, Igor Sokolov, and Ilyas Fatkhullin.
\newblock {EF21:} a new, simpler, theoretically better, and practically faster error feedback.
\newblock \emph{Advances in Neural Information Processing Systems}, 34:\penalty0 4384--4396, 2021.

\bibitem[Seide et~al.(2014)Seide, Fu, Droppo, Li, and Yu]{seide20141}
Frank Seide, Hao Fu, Jasha Droppo, Gang Li, and Dong Yu.
\newblock 1-bit stochastic gradient descent and its application to data-parallel distributed training of speech dnns.
\newblock In \emph{Interspeech}, volume 2014, pages 1058--1062. Singapore, 2014.

\bibitem[Shor(2012)]{shor2012minimization}
Naum~Zuselevich Shor.
\newblock \emph{Minimization methods for non-differentiable functions}, volume~3.
\newblock Springer Science \& Business Media, 2012.

\bibitem[Shulgin and Richt{\'a}rik(2024)]{shulgin2024convergence}
Egor Shulgin and Peter Richt{\'a}rik.
\newblock On the convergence of {DP-SGD} with adaptive clipping.
\newblock \emph{arXiv preprint arXiv:2412.19916}, 2024.

\bibitem[Stich and Karimireddy(2020)]{stich2019error}
Sebastian~U Stich and Sai~Praneeth Karimireddy.
\newblock The error-feedback framework: {SGD} with delayed gradients.
\newblock \emph{Journal of Machine Learning Research}, 21\penalty0 (237):\penalty0 1--36, 2020.

\bibitem[Stich et~al.(2018)Stich, Cordonnier, and Jaggi]{stich2018sparsified}
Sebastian~U Stich, Jean-Baptiste Cordonnier, and Martin Jaggi.
\newblock Sparsified {SGD} with memory.
\newblock \emph{Advances in neural information processing systems}, 31, 2018.

\bibitem[Sun et~al.(2019)Sun, Kairouz, Suresh, and McMahan]{sun2019can}
Ziteng Sun, Peter Kairouz, Ananda~Theertha Suresh, and H~Brendan McMahan.
\newblock Can you really backdoor federated learning?
\newblock \emph{arXiv preprint arXiv:1911.07963}, 2019.

\bibitem[Tang et~al.(2019)Tang, Yu, Lian, Zhang, and Liu]{tang2019doublesqueeze}
Hanlin Tang, Chen Yu, Xiangru Lian, Tong Zhang, and Ji~Liu.
\newblock Doublesqueeze: Parallel stochastic gradient descent with double-pass error-compensated compression.
\newblock In \emph{International Conference on Machine Learning}, pages 6155--6165. PMLR, 2019.

\bibitem[Vankov et~al.(2025)Vankov, Rodomanov, Nedich, Sankar, and Stich]{vankov2024optimizing}
Daniil Vankov, Anton Rodomanov, Angelia Nedich, Lalitha Sankar, and Sebastian~U Stich.
\newblock Optimizing $(l_0, l_1)$-smooth functions by gradient methods.
\newblock In \emph{The Thirteenth International Conference on Learning Representations}, 2025.

\bibitem[Wang et~al.(2023)Wang, Jayaraman, Evans, and Gu]{pmlr-v216-wang23b}
Lingxiao Wang, Bargav Jayaraman, David Evans, and Quanquan Gu.
\newblock Efficient privacy-preserving stochastic nonconvex optimization.
\newblock In Robin~J. Evans and Ilya Shpitser, editors, \emph{Proceedings of the Thirty-Ninth Conference on Uncertainty in Artificial Intelligence}, volume 216 of \emph{Proceedings of Machine Learning Research}, pages 2203--2213. PMLR, 31 Jul--04 Aug 2023.

\bibitem[Wang et~al.(2024)Wang, Zhou, Patel, Tang, and Saha]{wang2024efficient}
Lingxiao Wang, Xingyu Zhou, Kumar~Kshitij Patel, Lawrence Tang, and Aadirupa Saha.
\newblock Efficient private federated non-convex optimization with shuffled model.
\newblock In \emph{Privacy Regulation and Protection in Machine Learning}, 2024.

\bibitem[Wei et~al.(2020)Wei, Li, Ding, Ma, Yang, Farokhi, Jin, Quek, and Poor]{wei2020federated}
Kang Wei, Jun Li, Ming Ding, Chuan Ma, Howard~H Yang, Farhad Farokhi, Shi Jin, Tony~QS Quek, and H~Vincent Poor.
\newblock Federated learning with differential privacy: Algorithms and performance analysis.
\newblock \emph{IEEE transactions on information forensics and security}, 15:\penalty0 3454--3469, 2020.

\bibitem[Wu et~al.(2018)Wu, Huang, Huang, and Zhang]{wu2018error}
Jiaxiang Wu, Weidong Huang, Junzhou Huang, and Tong Zhang.
\newblock Error compensated quantized {SGD} and its applications to large-scale distributed optimization.
\newblock In \emph{International conference on machine learning}, pages 5325--5333. PMLR, 2018.

\bibitem[Yang et~al.(2022)Yang, Zhang, Chen, and Liu]{yang2022normalized}
Xiaodong Yang, Huishuai Zhang, Wei Chen, and Tie-Yan Liu.
\newblock Normalized/clipped {SGD} with perturbation for differentially private non-convex optimization.
\newblock \emph{arXiv preprint arXiv:2206.13033}, 2022.

\bibitem[Yu et~al.(2023)Yu, Jakovetic, and Kar]{yu2023smoothed}
Shuhua Yu, Dusan Jakovetic, and Soummya Kar.
\newblock Smoothed gradient clipping and error feedback for distributed optimization under heavy-tailed noise.
\newblock \emph{arXiv preprint arXiv:2310.16920}, 2023.

\bibitem[Zhang et~al.(2020{\natexlab{a}})Zhang, He, Sra, and Jadbabaie]{zhang2019gradient}
Jingzhao Zhang, Tianxing He, Suvrit Sra, and Ali Jadbabaie.
\newblock Why gradient clipping accelerates training: A theoretical justification for adaptivity.
\newblock In \emph{International Conference on Learning Representations}, 2020{\natexlab{a}}.

\bibitem[Zhang et~al.(2024)Zhang, Xie, and Yin]{zhang2024private}
Meifan Zhang, Zhanhong Xie, and Lihua Yin.
\newblock Private and communication-efficient federated learning based on differentially private sketches.
\newblock \emph{arXiv preprint arXiv:2410.05733}, 2024.

\bibitem[Zhang et~al.(2020{\natexlab{b}})Zhang, Fang, Liu, and Zhu]{zhang2020private}
Xin Zhang, Minghong Fang, Jia Liu, and Zhengyuan Zhu.
\newblock Private and communication-efficient edge learning: a sparse differential gaussian-masking distributed {SGD} approach.
\newblock In \emph{Proceedings of the Twenty-First International Symposium on Theory, Algorithmic Foundations, and Protocol Design for Mobile Networks and Mobile Computing}, pages 261--270, 2020{\natexlab{b}}.

\bibitem[Zhang et~al.(2022)Zhang, Chen, Hong, Wu, and Yi]{zhang2022understanding}
Xinwei Zhang, Xiangyi Chen, Mingyi Hong, Zhiwei~Steven Wu, and Jinfeng Yi.
\newblock Understanding clipping for federated learning: Convergence and client-level differential privacy.
\newblock In \emph{International Conference on Machine Learning, ICML 2022}, 2022.

\end{thebibliography}


\appendix

\tableofcontents

\section{Basic Facts}
For $n \in \mathbb{N}$ and $x_1,\ldots,x_n,x,y\in\R^d$, 
\begin{eqnarray}
    \inp{x}{y} & \leq & \norm{x}\norm{y}, \label{eqn:CS} \\
    \norm{x+y} & \leq & \norm{x} + \norm{y}, \quad \text{and}  \label{eqn:triangleIneq_1} \\ 
    \norm{\frac{1}{n}\sum_{i=1}^n x_i} & \leq & \frac{1}{n}\sum_{i=1}^n \norm{x_i}.\label{eqn:triangleIneq_2}
\end{eqnarray}

\section{Proof of Lemma~\ref{lemma:prop_norm_vr2}}
We prove the first statement by taking the Euclidean norm. Next, we prove the second statement. From the definition of the Euclidean norm, 
 \begin{eqnarray*}
     \sqnorm{g -   \beta\Normalize{\alpha}{g}} 
     & \overset{\eqref{eqn:normalize}}{=} & \sqnorm{g} + \frac{\beta^2}{(\alpha+\norm{g})^2}\sqnorm{g} - 2\beta \frac{\sqnorm{g}}{\alpha+\norm{g}} \\
     & = & \left( 1 - \frac{\beta}{\alpha + \norm{g}} \right)^2 \sqnorm{g}.
 \end{eqnarray*}

\section{Comparison of EF21 with Clipping and Smoothed Normalization}

We compare the modified \algname{EF21}  mechanism, where a contractive compressor is replaced with clipping in \algname{Clip21} and with smoothed normalization in \alg.
To compare these modified updates, given the optimal vector $g^\star\in\R^d$,  we consider the single-node  \algname{EF21} mechanism, which computes the memory vector $g^k\in\R^d$ according to
\begin{eqnarray}
    g^{k+1} = g^k + \Psi(g^\star - g^k),
\end{eqnarray}
where $\Psi:\R^d\rightarrow\R^d$ is the biased gradient estimator, and $g^0\in\R^d$ is the initial memory vector. 

If $\Psi(g) = \Clip{\tau}{g}$, then from Theorem 4.3 of \citet{khirirat2023clip21}  
\begin{eqnarray*}
    \norm{g^{k}-g^\star} \leq \max(0, \norm{g^0-g^\star}-k\tau).
\end{eqnarray*}

If $\Psi(g) = \Normalize{\alpha}{g}$, then from Lemma~\ref{lemma:prop_norm_vr2}
\begin{eqnarray*}
    \sqnorm{g^\star-g^k}
    & = & \sqnorm{g^\star-g^{k-1} -  \beta\Normalize{\alpha}{g^\star - g^{k-1}}}  \\
    & = & \left( 1- \frac{\beta}{\alpha + \norm{g^\star-g^{k-1}}} \right)^2 \sqnorm{g^\star-g^{k-1}}  \\
    & \vdots &   \\
    & = & \sqnorm{g^\star-g^{0}} \cdot \prod_{l=1}^k \left( 1- \frac{\beta}{\alpha + \norm{g^\star-g^{l-1}}} \right)^2.
\end{eqnarray*}

In conclusion, while the  \algname{EF21}  mechanism with clipping ensures that the memory $g^k$ will reach $g^\star$ within a finite number of iterations $k$ (when $k \geq \norm{g^0-g^\star}/\tau$),  the  \algname{EF21}  mechanism with smoothed normalization guarantees that $g^k$ will eventually reach $g^\star$ (provided that$\beta/\alpha < 1$).

\section{Proof of Theorem~\ref{thm:norm_ef21_nondp}}

We prove  Theorem~\ref{thm:norm_ef21_nondp} by  Lemma~\ref{lemma:norm21_approach2}, which states $\norm{\nabla f_i(x^{k+1})-g_i^{k+1}} \leq R$ for some positive scalars $R$, given that $\norm{\nabla f_i(x^{k})-g_i^{k}} \leq R$, and hyperparameters $\alpha,\beta,\gamma$ are properly tuned. 

\begin{lemma}[Non-private setting]\label{lemma:norm21_approach2}
Consider \alg (Algorithm~\ref{alg:norm_norm21_dp_v1})  for solving Problem~\eqref{eqn:problem}, where Assumption~\ref{assum:lowebound_f} holds.
If $\norm{\nabla f_i(x^k)-g_i^{k}} \leq R$, $\frac{\beta}{\alpha + R} < 1$, and $\gamma \leq \frac{\beta R}{\alpha + R} \frac{1}{L_{\max}}$ with $L_{\max}=\max_{i\in[1,n]} L_i$, then $  \norm{\nabla f_i(x^{k+1})-g_i^{k+1}}  \leq R$.
\end{lemma}
\begin{proof}
From the definition of the Euclidean norm,  
\begin{eqnarray*}
    \norm{\nabla f_i(x^{k+1})-g_i^{k+1}} 
    & \overset{\eqref{eqn:triangleIneq_1}}{\leq} & \norm{\nabla f_i(x^{k+1})-\nabla f_i(x^k)} + \norm{\nabla f_i(x^k)-g_i^{k+1}} \\
    & \overset{g_i^{k+1}}{=} & \norm{\nabla f_i(x^{k+1})-\nabla f_i(x^k)} \\
    && + \norm{\nabla f_i(x^k)-g_i^{k}-\beta\Normalize{\alpha}{\nabla f_i(x^k)-g_i^{k}}} \\
    & \overset{\text{Lemma~\ref{lemma:prop_norm_vr2}}}{\leq} &  \norm{\nabla f_i(x^{k+1})-\nabla f_i(x^k)} \\
    &&+  \left\vert 1 - \frac{\beta}{\alpha + \norm{\nabla f_i(x^k)-g_i^{k}}} \right\vert \norm{\nabla f_i(x^k)-g_i^{k}} \\
    & \overset{{\text{Assumption~\ref{assum:lowebound_f}, and }x^{k+1}}}{\leq} &  L_{\max}\gamma +  \left\vert 1 - \frac{\beta}{\alpha + \norm{\nabla f_i(x^k)-g_i^{k}}} \right\vert \norm{\nabla f_i(x^k)-g_i^{k}}. 
\end{eqnarray*}

If $\norm{\nabla f_i(x^k)-g_i^{k}} \leq R$ and $\frac{\beta}{\alpha + R} < 1$, then 
$\norm{\nabla f_i(x^{k+1})-g_i^{k+1}}  \leq R$ when 
\begin{eqnarray*}
    \gamma \leq \frac{\beta R}{\alpha + R} \frac{1}{L_{\max}}. 
\end{eqnarray*}
\end{proof}

Now, we are ready to prove the result in Theorem~\ref{thm:norm_ef21_nondp} in four steps. 

\paragraph{Step 1) Prove by induction that $\norm{\nabla f_i(x^k)-g_i^{k}} \leq R$ for $R = \max_{i\in[1,n]}\norm{\nabla f_i(x^0)-g_i^{0}}$.} For $k=0$, this is obvious. Next, let $\norm{\nabla f_i(x^l)-g_i^{l}} \leq R$ for $R = \max_{i\in[1,n]}\norm{\nabla f_i(x^0)-g_i^{0}}$ for $l=0,1,\ldots,k$. Then, if $\beta/(\alpha + R) < 1$, and $\gamma \leq \frac{\beta R}{\alpha + R}\frac{1}{L_{\max}}$, then 
from Lemma~\ref{lemma:norm21_approach2} $\norm{\nabla f_i(x^{k+1})-g_i^{k+1}} \leq R$.

\paragraph{Step 2) Bound $\norm{\nabla f_i(x^k)-g_i^{k+1}}$.}

From the definition of the Euclidean norm, 
\begin{eqnarray*}
\norm{\nabla f_i(x^k)-g_i^{k+1}} 
& \overset{g_i^{k+1}}{=} & \norm{\nabla f_i(x^k) - g_i^{k} -  \beta\Normalize{\alpha}{\nabla f_i(x^k)-g_i^{k} }} \\
& \overset{\text{Lemma~\ref{lemma:prop_norm_vr2}}}{\leq} & \left\vert 1 - \frac{\beta}{\alpha + \norm{\nabla f_i(x^k)-g_i^{k}}} \right\vert \norm{\nabla f_i(x^k)-g_i^{k}} \\
& \overset{(*)}{\leq} & \left( 1 - \frac{\beta}{\alpha + R } \right) R \leq R,
\end{eqnarray*}
where we reach $(*)$ by the fact that $\norm{\nabla f_i(x^k)-g_i^{k}} \leq R$, $\frac{\beta}{\alpha +R} < 1$, and $\gamma \leq \frac{\beta R}{\alpha+R}\frac{1}{L_{\max}}$.

\paragraph{Step 3) Derive the descent inequality.}
By the $L$-smoothness of $f$, by the definition of $x^{k+1}$, and by the fact that $\hat g^{k+1} = g^{k+1}$,
\begin{eqnarray*}
     f(x^{k+1})-f^{\inf} 
    & \leq & f(x^k)-f^{\inf} - \frac{\gamma}{\norm{ g^{k+1}}} \inp{\nabla f(x^k)}{ g^{k+1}} + \frac{L\gamma^2}{2} \\
    & = & f(x^k)-f^{\inf} - \gamma \norm{ g^{k+1}} + \frac{\gamma}{\norm{ g^{k+1}}} \inp{\nabla f(x^k)-g^{k+1}}{g^{k+1}} + \frac{L\gamma^2}{2} \\
    & \overset{\eqref{eqn:CS}}{\leq} & f(x^k)-f^{\inf} - \gamma \norm{g^{k+1}} + \gamma\norm{\nabla f(x^k)-g^{k+1}} + \frac{L\gamma^2}{2} \\
    & \overset{\eqref{eqn:triangleIneq_1}}{\leq} & f(x^k)-f^{\inf} - \gamma \norm{\nabla f(x^k)} + 2\gamma\norm{\nabla f(x^k)-g^{k+1}} + \frac{L\gamma^2}{2} \\
    & \overset{\eqref{eqn:triangleIneq_2}}{\leq} & f(x^k)-f^{\inf} - \gamma \norm{\nabla f(x^k)} + 2\gamma \frac{1}{n}\sum_{i=1}^n \norm{\nabla f_i(x^k)-g_i^{k+1}} + \frac{L\gamma^2}{2}.
\end{eqnarray*}
Next, since $\norm{\nabla f_i(x^k)-g_i^{k+1}} \leq R$ with $R = \max_{i\in[1,n]}\norm{\nabla f_i(x^0)-g_i^{0}}$, we have 
\begin{eqnarray*}
    f(x^{k+1})-f^{\inf}  \leq f(x^k)-f^{\inf} - \gamma \norm{\nabla f(x^k)} + 2\gamma \max_{i\in[1,n]} \norm{\nabla f_i(x^0)-g_i^{0}} + \frac{L\gamma^2}{2}. 
\end{eqnarray*}

\paragraph{Step 4) Finalize the convergence rate.}
Finally, by re-arranging the terms of the  inequality,
\begin{align*}
    \min_{k\in[0,K]} \norm{\nabla f(x^k)}
    & \leq  \frac{1}{K+1} \sum_{k=0}^K \norm{\nabla f(x^k)} \\
    & \leq  \frac{[f(x^0)-f^{\inf}]-[f(x^{K+1})-f^{\inf}]}{\gamma(K+1)} + 2\max_{i\in[1,n]} \norm{\nabla f_i(x^0)-g_i^{0}} + \frac{L}{2}\gamma \\
    & \overset{(\dagger)}{\leq} \frac{f(x^0)-f^{\inf}}{\gamma(K+1)} + 2\max_{i\in[1,n]} \norm{\nabla f_i(x^0)-g_i^{0}} + \frac{L}{2}\gamma,
\end{align*}
where we reach $(\dagger)$ by the fact that $f^{\inf} \geq f(x^{K+1})$.

\section{Proof of Corollary~\ref{corr:norm_ef21_nondp}}
If $g_i^{0}\in\R^d$ is chosen such that $\max_{i\in[1,n]} \norm{\nabla f_i(x^0)-g_i^{0}}  = \frac{D}{(K+1)^{1/2}}$ with any $D>0$, $\gamma \leq \frac{\beta}{L_{\max}} \frac{D}{\alpha+D} \frac{1}{(K+1)^{1/2}}$, and $\beta < \alpha$, then from Theorem~\ref{thm:norm_ef21_nondp}, we obtain $\gamma \leq \frac{\beta R}{\alpha + R}\frac{1}{L_{\max}}$ with $R=\max_{i\in[1,n]}\norm{\nabla f_i(x^0)-g_i^{0}}$, and  
\begin{eqnarray*}
     \min_{k\in[0,K]} \norm{\nabla f(x^k)}  & \leq & \frac{L_{\max} (\alpha + D)}{\beta D} \frac{f(x^0)-f^{\inf}}{(K+1)^{1/2}} + 2 \frac{D}{(K+1)^{1/2}} + \frac{L}{2} \frac{\beta D}{L_{\max}(\alpha+D)} \frac{1}{(K+1)^{1/2}}.
\end{eqnarray*}

\section{\alg and \algname{Clip21} Comparison}\label{app:compare_norm21_clip21}

We show that the convergence bound of \alg~\eqref{eqn:final_bound_norm21} has a smaller factor than that of \algname{Clip21} from  Theorem 5.6. of \citet{khirirat2023clip21}.

To show this, let $\hat x^K$ be selected uniformly at random from a set $\{x^0,x^1,\ldots,x^K\}$. Then, from Theorem 5.6. of \citet{khirirat2023clip21}, \algname{Clip21} achieves the following convergence bound:
\begin{eqnarray*}
    \min_{k\in[0,K]} \norm{\nabla f(x^k)} & \leq& \Exp{\norm{\nabla f(\hat x^K)}}  \\
    & \leq & \sqrt{\Exp{\sqnorm{\nabla f(\hat x^K)}}} \\
    & \leq & \frac{L_{\max}(f(x^0)-f^{\inf})}{\tau (K+1)^{1/2}}  + \frac{\sqrt{(1+C_1/\tau)C_2}}{(K+1)^{1/2}},
\end{eqnarray*}
where $\tau>0$ is a clipping threshold, $C_1 = \max_{i\in[1,n]}\norm{\nabla f_i(x^0)}$, and $C_2 = \max(\max(L,L_{\max})(f(x^0)-f^{\inf})), C_1^2)$.

If $\tau = \frac{L_{\max}}{\sqrt{2L}}\sqrt{f(x^0)-f^{\inf}}$, then 
\begin{eqnarray*}
    \min_{k\in[0,K]} \norm{\nabla f(x^k)} 
    & \leq & \sqrt{\frac{2L (f(x^0)-f^{\inf})}{K+1}}  + \frac{\sqrt{ \left(1+ \frac{C_1\sqrt{2L}}{L_{\max}\sqrt{f(x^0)-f^{\inf}}} \right)C_2}}{(K+1)^{1/2}} \\
    & \leq & \sqrt{\frac{2L (f(x^0)-f^{\inf})}{K+1}} \\
    &&+ \frac{\sqrt{C_2} + \cO\left( {\max}(\sqrt{C_1}\sqrt[4]{f(x^0)-f^{\inf}}, C_1^3/\sqrt{f(x^0)-f^{\inf}}) \right) }{(K+1)^{1/2}}.
\end{eqnarray*}
The first term in the convergence bound of \algname{Clip21} matches that of \alg as given in \eqref{eqn:final_bound_norm21}. However, the second term in the convergence bound of \alg is ${D}/{\sqrt{K+1}}$, where \( D > 0 \) can be made arbitrarily small. In contrast, the corresponding term for \algname{Clip21} is \( C/\sqrt{K+1} \), where \( C > 0 \) may become significantly larger than $D$ if \( x^0 \in \mathbb{R}^d \) is far from the stationary point, leading to a large value of \( C_1 =\max_{i\in[1,n]}\norm{\nabla f_i(x^0)} \).

\section{Proof of Theorem~\ref{thm:dp_n_norm21}}

We prove Theorem~\ref{thm:dp_n_norm21} by  two useful lemmas: 
\begin{enumerate}
    \item Lemma~\ref{lemma:norm21_approach2}, which states  $\norm{\nabla f_i(x^{k+1})-g_i^{k+1}} \leq R$  for some positive scalars $R$, given that $\norm{\nabla f_i(x^{k})-g_i^{k}} \leq R$ and the hyperparameters $\gamma,\beta,\alpha$ are properly tuned, and 
    \item Lemma~\ref{lemma:dp_norm21_weight_tricks}, which bounds the difference in expectation between the memory vectors maintained by the central server and clients.   
\end{enumerate}

\begin{lemma}[DP setting]\label{lemma:dp_norm21_weight_tricks}
Consider \algname{DP-}\alg (Algorithm~\ref{alg:norm_norm21_dp_v1})  for solving Problem~\eqref{eqn:problem}, where Assumption~\ref{assum:lowebound_f} holds.
If $\hat g^0 = \frac{1}{n}\sum_{i=1}^n g_i^0$, then 
\begin{eqnarray*}
\Exp{\norm{\hat g^{k+1} - \frac{1}{n}\sum_{i=1}^n g^{k+1}_i}}  \leq \sqrt{{\color{black}\frac{{\beta^2(K+1) \sigma_{\rm DP}^2}}{n}} }.
\end{eqnarray*}
\end{lemma}
\begin{proof}
From the definition of $g^k_i$ and $\hat g^k$, 
\begin{eqnarray*}
    e^{k+1} = e^{k} + \beta z^{k+1},
\end{eqnarray*}
where $e^k = \hat g^k - \frac{1}{n}\sum_{i=1}^n g^k_i$, and $z^k = \frac{1}{n}\sum_{i=1}^n z_i^k$. 
By applying the equation recursively, 
\begin{eqnarray*}
    e^{k+1} = e^{0} + \beta \sum_{l=1}^{k+1} z^l.
\end{eqnarray*}
Therefore, by the triangle inequality, 
\begin{eqnarray*}
    \norm{e^{k+1}} \leq \norm{e^{0}} +  \norm{\beta \sum_{l=1}^{k+1} z^l}.
\end{eqnarray*}

If $\hat g^{0}= \frac{1}{n}\sum_{i=1}^n g_i^{0}$, then $e^0=0$ and therefore
\begin{eqnarray*}
    \norm{e^{k+1}} \leq   \norm{\beta \sum_{l=1}^{k+1} z^l}.
\end{eqnarray*}
By taking the expectation,
\begin{eqnarray*}
    \Exp{\norm{e^{k+1}}} &\leq& \Exp{\norm{\beta \sum_{l=1}^{k+1} z^l}} \\
    & = & \Exp{\sqrt{ \sqnorm{\beta \sum_{l=1}^{k+1} z^l}  }} \\
    & {\leq} & \sqrt{\Exp{ \sqnorm{\beta \sum_{l=1}^{k+1} z^l}  }},
\end{eqnarray*}
where we reach the last inequality by Jensen's inequality. Next, by expanding the terms, 
\begin{eqnarray*}
    \Exp{\norm{e^{k+1}}} 
    & \leq & \sqrt{ \beta^2 \sum_{l=1}^{k+1} \Exp{\sqnorm{z^l}} + \beta^2 \sum_{j\neq i} \Exp{\inp{z^i}{z^j}}  } \\
    & \overset{(*)}{=} &\sqrt{ \beta^2 \sum_{l=1}^{k+1} \Exp{\sqnorm{z^l}}} \\
    & \overset{(\ddagger)}{\leq} & \sqrt{\frac{\beta^2}{n} \sum_{l=1}^{k+1}\sigma_{\rm DP}^2 },
\end{eqnarray*}
where we reach $(*)$ by the fact that  $\Exp{\langle z^j , z^i\rangle}=0$ for $i \neq j$, and $(\ddagger)$ by the fact that $\Exp{\sqnorm{z^k}} = {\sigma_{\rm DP}^2}/{n}$ (as $z_i^k$ is independent of $z_j^k$ for $i\neq j$).
Therefore, 
\begin{eqnarray*}
\Exp{\norm{e^{k+1}}} 
& \leq & \sqrt{\frac{\beta^2 (k+1)\sigma_{\rm DP}^2}{n}} \\
& \overset{k \leq K}{\leq} & \sqrt{ \frac{\beta^2(K+1) \sigma_{\rm DP}^2}{n} }.
\end{eqnarray*}
\end{proof}

Now, we prove Theorem~\ref{thm:dp_n_norm21} in three steps.

\paragraph{Step 1) Prove by induction that $\norm{\nabla f_i(x^k)-g_i^{k}} \leq R$ for $R = \max_{i\in[1,n]}\norm{\nabla f_i(x^0)-g_i^{0}}$.} For $k=0$, this is obvious. Next, let $\norm{\nabla f_i(x^l)-g_i^{l}} \leq R$ for $R = \max_{i\in[1,n]}\norm{\nabla f_i(x^0)-g_i^{0}}$ for $l=0,1,\ldots,k$. Then, if $\beta/(\alpha + R) < 1$, and $\gamma \leq \frac{\beta R}{\alpha + R}\frac{1}{L_{\max}}$, then 
from Lemma~\ref{lemma:norm21_approach2} $\norm{\nabla f_i(x^{k+1})-g_i^{k+1}} \leq R$. 

\paragraph{Step 2) Bound $\norm{\nabla f_i(x^k)-g_i^{k+1}}$.}

From the definition of the Euclidean norm, 
\begin{eqnarray*}
\norm{\nabla f_i(x^k)-g_i^{k+1}} 
& \overset{g_i^{k+1}}{=} & \norm{\nabla f_i(x^k) - g_i^{k} -  \beta\Normalize{\alpha}{\nabla f_i(x^k)-g_i^{k} }} \\
& \overset{\text{Lemma~\ref{lemma:norm21_approach2}}}{\leq} & \left\vert 1 - \frac{\beta}{\alpha + \norm{\nabla f_i(x^k)-g_i^{k}}} \right\vert \norm{\nabla f_i(x^k)-g_i^{k}}.
\end{eqnarray*}

\paragraph{Step 3) Derive the descent inequality in $\Exp{f(x^k)-f^{\inf}}$.}
Denote $g^k = \frac{1}{n}\sum_{i=1}^n g_i^k$. By the $L$-smoothness of $f$, and by the definition of $x^{k+1}$, 
\begin{eqnarray*}
     f(x^{k+1})-f^{\inf} 
    & \leq & f(x^k)-f^{\inf} - \frac{\gamma}{\norm{\hat g^{k+1}}} \inp{\nabla f(x^k)}{\hat g^{k+1}} + \frac{L\gamma^2}{2} \\
    & = & f(x^k)-f^{\inf} - \gamma \norm{\hat g^{k+1}} + \frac{\gamma}{\norm{\hat g^{k+1}}} \inp{\nabla f(x^k)-\hat g^{k+1}}{\hat g^{k+1}} + \frac{L\gamma^2}{2} \\
    & \overset{\eqref{eqn:CS}}{\leq} & f(x^k)-f^{\inf} - \gamma \norm{\hat g^{k+1}} + \gamma\norm{\nabla f(x^k)-\hat g^{k+1}} + \frac{L\gamma^2}{2} \\
    & \overset{\eqref{eqn:triangleIneq_1}}{\leq} & f(x^k)-f^{\inf} - \gamma \norm{\nabla f(x^k)} + 2\gamma\norm{\nabla f(x^k)-\hat g^{k+1}} + \frac{L\gamma^2}{2} \\
    & \overset{\eqref{eqn:triangleIneq_2}}{\leq} & f(x^k)-f^{\inf} - \gamma \norm{\nabla f(x^k)} + 2\gamma \frac{1}{n}\sum_{i=1}^n \norm{\nabla f_i(x^k)-g_i^{k+1}} \\
    && + 2\gamma\norm{\hat g^{k+1}- g^{k+1}} + \frac{L\gamma^2}{2}. 
\end{eqnarray*}
Next, let $\cF^k$ be the history up to iteration $k$, i.e. $\cF^k \eqdef \{x^0,z_1^0,\ldots,z_n^0,\ldots,x^k,z_1^k,\ldots,z_n^k\}$. Then, 
\begin{eqnarray*}
    \ExpCond{f(x^{k+1})-f^{\inf}}{\cF^k} 
    & \leq & f(x^k)-f^{\inf} - \gamma \norm{\nabla f(x^k)} + 2\gamma \frac{1}{n}\sum_{i=1}^n  \ExpCond{\norm{\nabla f_i(x^k)-g_i^{k+1}}}{\cF^k} \\
   && + 2\gamma \ExpCond{\norm{\hat g^{k+1}- g^{k+1}}}{\cF^k} + \frac{L\gamma^2}{2}. 
\end{eqnarray*}
Next, by the upper-bound for $\norm{\nabla f_i(x^k)-g_i^{k+1}}$, 
\begin{eqnarray*}
    \ExpCond{\norm{\nabla f_i(x^k)-g_i^{k+1}}}{\cF^k} 
    & \leq &  \ExpCond{\left\vert 1 - \frac{\beta}{\alpha + \norm{\nabla f_i(x^k)-g_i^{k}}} \right\vert \norm{\nabla f_i(x^k)-g_i^{k}}}{\cF^k}  \\
    & = & \left\vert 1 - \frac{\beta}{\alpha + \norm{\nabla f_i(x^k)-g_i^{k}}} \right\vert \norm{\nabla f_i(x^k)-g_i^{k}} \\
    & \leq & \left(1- \frac{\beta}{\alpha+R} \right) R \leq R,
\end{eqnarray*}
where we reach the second inequality by the fact that $\norm{\nabla f_i(x^k)-g_i^{k}} \leq R$, $\frac{\beta}{\alpha +R} < 1$, and $\gamma \leq \frac{\beta R}{\alpha+R}\frac{1}{L_{\max}}$. Thus, 
\begin{eqnarray*}
    \ExpCond{f(x^{k+1})-f^{\inf}}{\cF^k} 
    & \leq & f(x^k)-f^{\inf} - \gamma \norm{\nabla f(x^k)} + 2\gamma R \\
   && + 2\gamma \ExpCond{\norm{\hat g^{k+1}- g^{k+1}}}{\cF^k} + \frac{L\gamma^2}{2}. 
\end{eqnarray*}
By taking the expectation, and by the tower property $\Exp{\ExpCond{X}{Y}}=\Exp{X}$,
\begin{eqnarray*}
    \Exp{f(x^{k+1})-f^{\inf}}
    & = & \Exp{ \ExpCond{f(x^{k+1})-f^{\inf}}{\cF^k} } \\
    & \leq & \Exp{f(x^k)-f^{\inf}} - \gamma \Exp{\norm{\nabla f(x^k)}} + 2\gamma R  \\
    && + 2\gamma \Exp{\norm{\hat g^{k+1}- g^{k+1}}} + \frac{L\gamma^2}{2}.
\end{eqnarray*}
Next, by using Lemma~\ref{lemma:dp_norm21_weight_tricks}, 
\begin{eqnarray*}
  \Exp{ f(x^{k+1})-f^{\inf} } & \leq &   \Exp{f(x^k)-f^{\inf}} - \gamma \Exp{\norm{\nabla f(x^k)}} + 2\gamma R \\
  && + 2\gamma {\color{black}\sqrt{ \frac{\beta^2(K+1) \sigma_{\rm DP}^2}{n} }}+ \frac{L\gamma^2}{2}.
\end{eqnarray*}
Therefore, 
\begin{eqnarray*}
   && \min_{k\in[0,K]} \Exp{\norm{\nabla f(x^k)}}
     \leq  \frac{1}{K+1} \sum_{k=0}^K \Exp{\norm{\nabla f(x^k)}} \\
    && \leq  \frac{\Exp{f(x^0)-f^{\inf}}-\Exp{f(x^{K+1})-f^{\inf}}}{\gamma(K+1)}  + 2 R + 2 {\color{black}\sqrt{ \frac{\beta^2(K+1) \sigma_{\rm DP}^2}{n} }}+ \frac{L}{2}\gamma \\
    && \leq  \frac{f(x^0)-f^{\inf}}{\gamma(K+1)}   + 2 R+ 2 {\color{black}\sqrt{ \frac{\beta^2(K+1) \sigma_{\rm DP}^2}{n} }} + \frac{L}{2}\gamma,
\end{eqnarray*}
where  we reach the last inequality by the fact that $f^{\inf} \geq f(x^{K+1})$.

\section{Proof of Corollary~\ref{corr:dp_n_norm21}}

Let $\sigma_{\rm DP} = \cO\left( \frac{\sqrt{(K+1)\log(1/\delta)}}{\epsilon} \right)$,  
and  $\beta = \frac{\beta_0}{K+1}$ with $0<\beta_0<\alpha + R$. Then, from Theorem~\ref{thm:dp_n_norm21}, we get$\gamma \leq \frac{\beta_0 R}{\alpha + R}\frac{1}{L_{\max}} \frac{1}{K+1}$ with $R = \max_{i\in[1,n]}\norm{\nabla f_i(x^0)-g_i^{0}}$, and 
\begin{eqnarray*}
  \min_{k\in[0,K]} \Exp{\norm{\nabla f(x^k)}} 
  &\leq & \frac{L_{\max}(\alpha +R)(f(x^0)-f^{\inf})}{\beta_0 R}  + 2 R + \cO\left(\frac{\beta_0  \sqrt{\log(1/\delta)}}{ {\color{black}\sqrt{n}} \epsilon} \right) \\
  && + \frac{L \beta_0 R}{2(\alpha + R)L_{\max}} \frac{1}{K+1}.
\end{eqnarray*}

If $\beta_0 \leq \sqrt{\frac{L_{\max}(\alpha +R)(f(x^0)-f^{\inf})}{R}} \frac{{\color{black}\sqrt[4]{n}}\sqrt{\epsilon}}{ \sqrt[4]{d} \sqrt[4]{\log(1/\delta)}}$ and $\alpha > \beta_0$, then
\begin{eqnarray*}
  \min_{k\in[0,K]} \Exp{\norm{\nabla f(x^k)}} &\leq&   2 R + \cO\left(\sqrt{\frac{L_{\max}(\alpha +R)(f(x^0)-f^{\inf})}{R}}\frac{ \sqrt[4]{d} \sqrt[4]{\log(1/\delta)}}{ {\color{black}\sqrt[4]{n}}\sqrt{\epsilon}} \right) \\
  &&+ \cO\left( \frac{1}{K+1} \right).
\end{eqnarray*}

\newpage

\section{Experimental Details and Additional Results} \label{app:experiments}

We include details on experimental setups and additional results in the non-private and private training for the ResNet20 model on the CIFAR-10 dataset.  

\subsection{Additional Experimental Details}
The dataset was split into train (90\%) and test (10\%) sets. The train samples were randomly shuffled and distributed across 10 workers. Every worker computed gradients with batch size 32. The training was performed for 300 communication rounds. The random seed was fixed to 42 for reproducibility.

All the methods were run with a constant step size (learning rate) without other techniques, such as schedulers, warm-up, or weight decay. 
They were evaluated across the following hyperparameter combinations: 
\begin{itemize}
    \item step size $\gamma$: $\{0.001, 0.01, 0.1, 1.0\}$,
    \item Sensitivity/clip threshold $\beta$: $\{0.01, 0.1, 1.0, 10.0\}$,
    \item Smoothed normalization value $\alpha$: $\{0.01, 0.1, 1.0\}$.
\end{itemize}

Our implementation is based on the public GitHub repository of \citet{idelbayev18a}. Experiments were performed on a machine with a single GPU: NVIDIA GeForce RTX 3090.

\subsection{Non-private Training}

\subsubsection{Sensitivity of \alg to Parameters $\beta, \alpha$} \label{app:exp_sensitivity}

\alg provides stable convergence with respect to $\alpha$ and robust convergence for different $\beta$ values, in terms of the minimal training loss, the final train loss, and the final test accuracy in Figure~\ref{fig:heatmap_non_private_other}.

\begin{figure}[h]
\centering
\begin{minipage}[c]{0.32\textwidth}
    \centering
    \includegraphics[width=\linewidth]{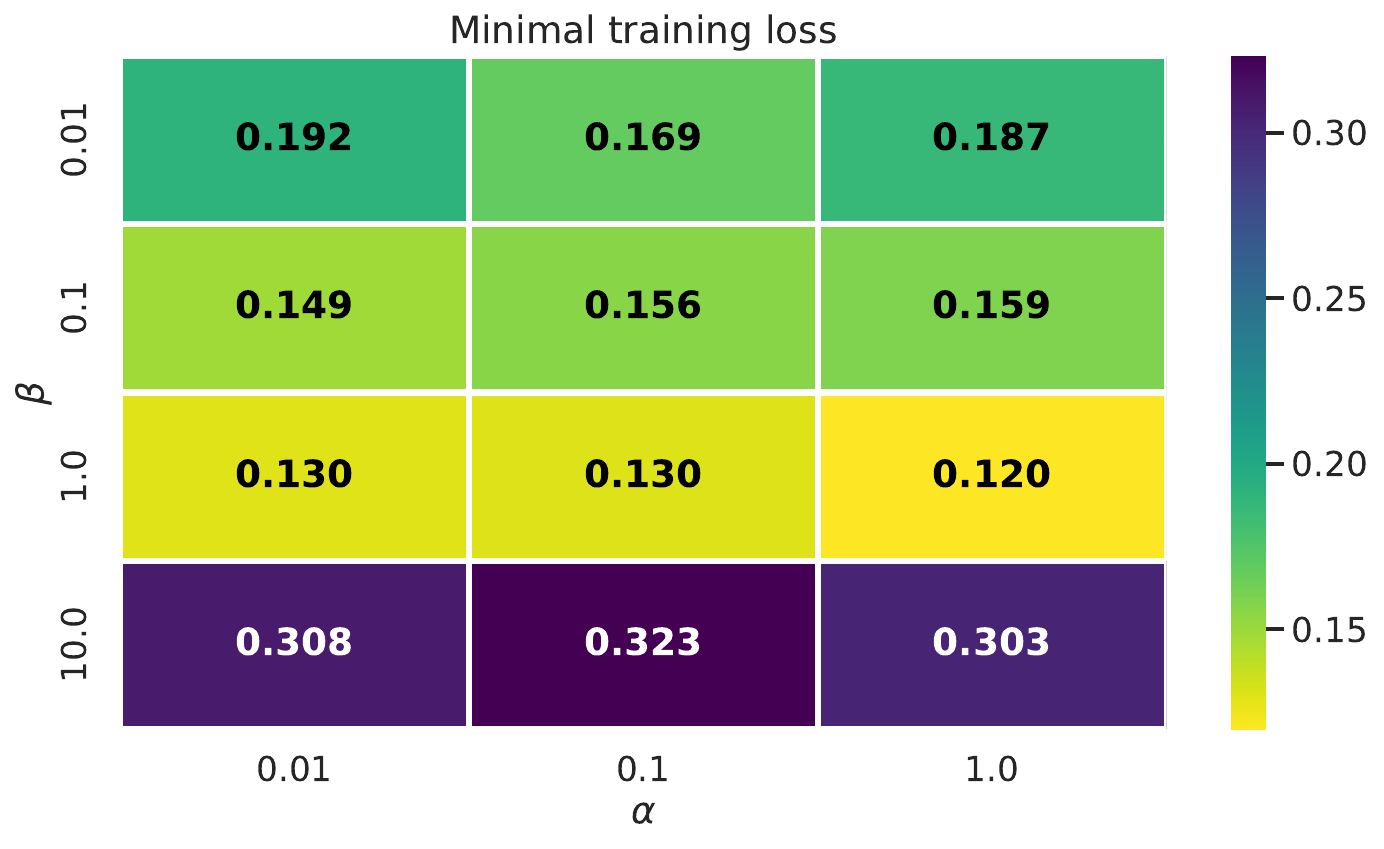}
    \label{fig:loss_heatmap}
\end{minipage}
\hfill
\begin{minipage}[c]{0.32\textwidth}
    \centering
    \includegraphics[width=\linewidth]{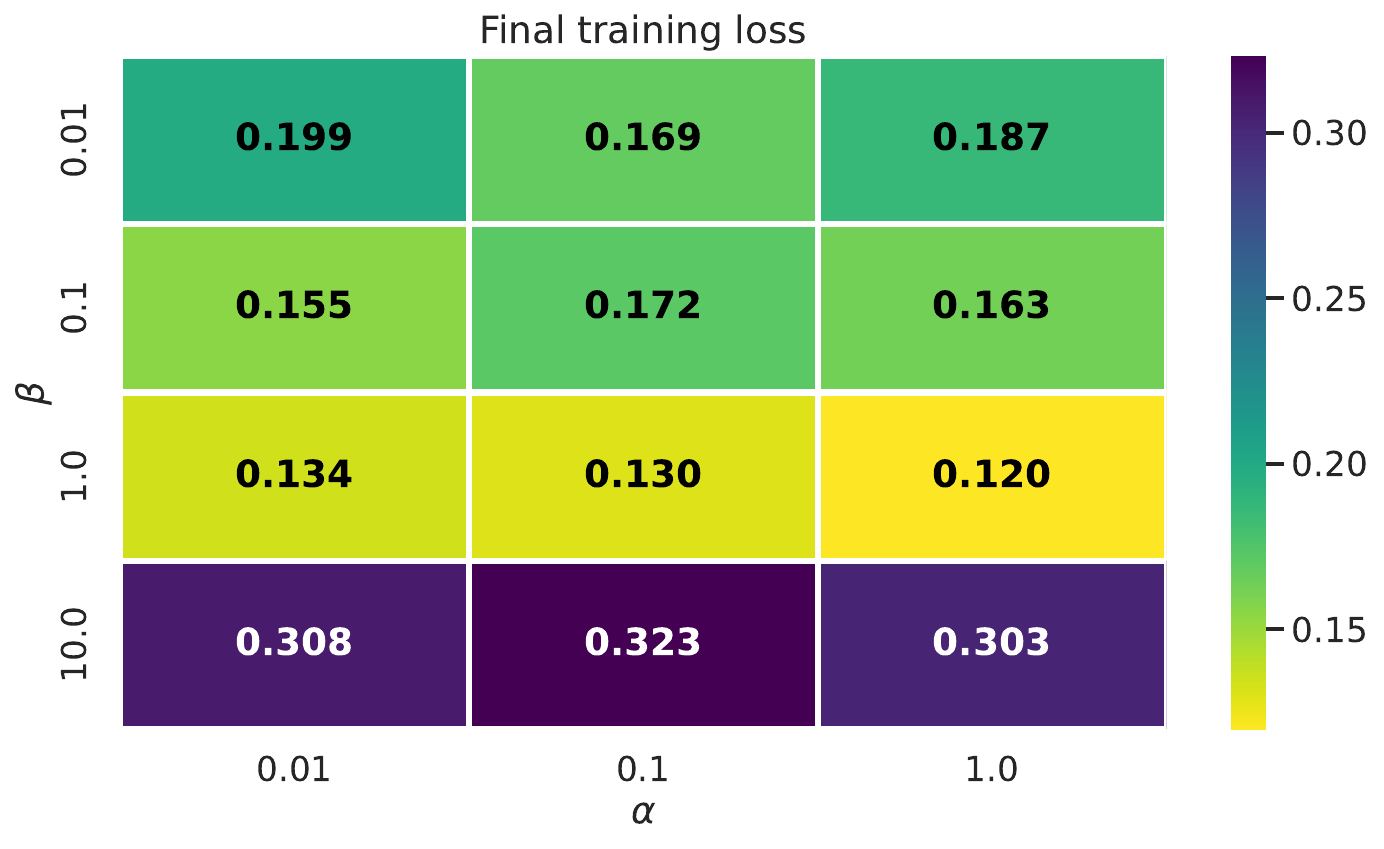}
    \label{fig:final_accuracy_heatmap}
\end{minipage}
\hfill
\begin{minipage}[c]{0.32\textwidth}
    \centering
    \includegraphics[width=\linewidth]{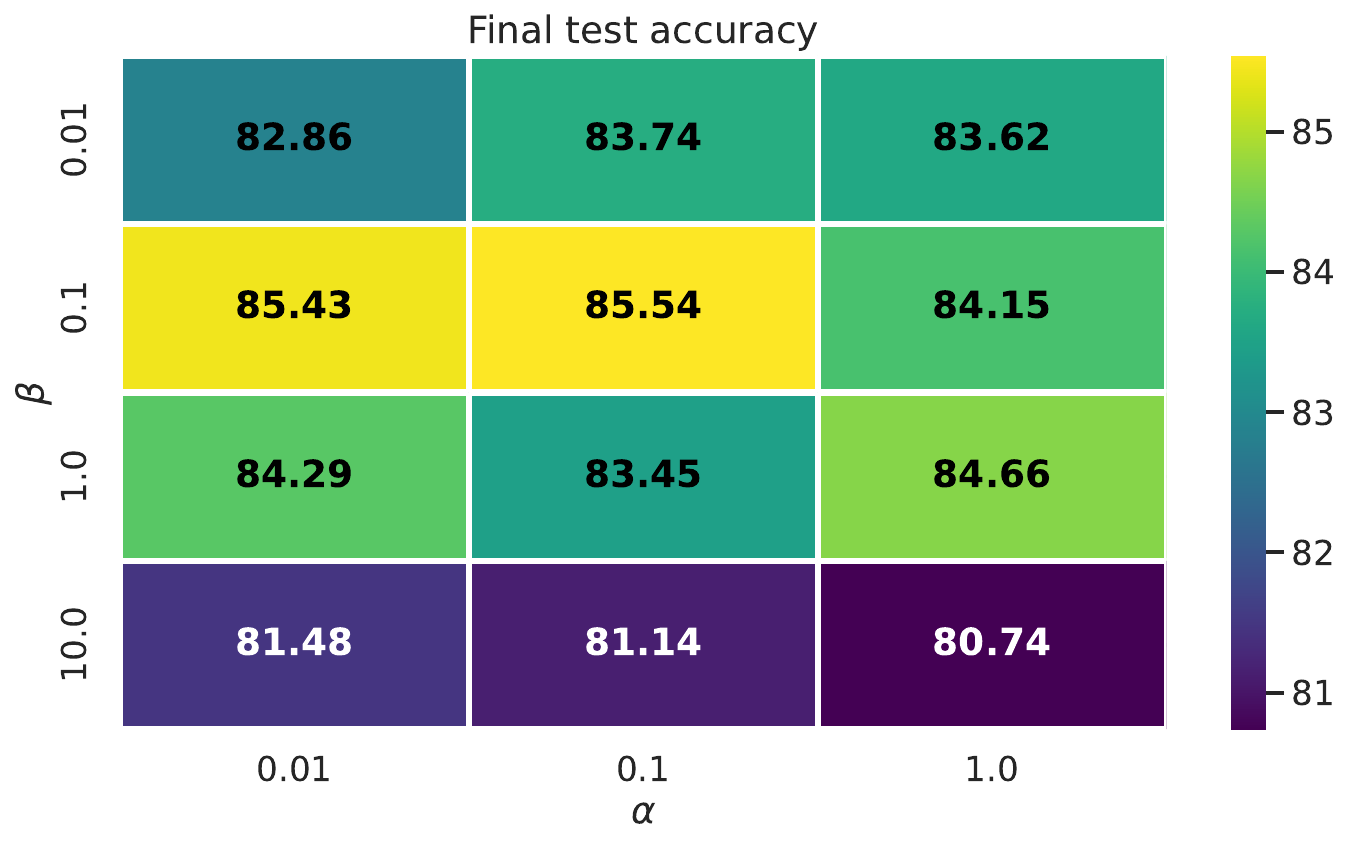}
    \label{fig:final_loss_heatmap}
\end{minipage}
\caption{\textbf{Minimal} train loss (left), \textbf{final} train loss (middle), and \textbf{final} test accuracy (right) by non-private \alg, after 300 communication rounds using a fine-tuned constant step size $\gamma$.} 
\label{fig:heatmap_non_private_other}
\end{figure}

\newpage

\begin{wrapfigure}{r}{0.5\textwidth}  
\vspace{-35pt}
\centering
\begin{tabular}{cccc}
\toprule
Method & $\beta$ & $\gamma$ & Final Accuracy \\
\midrule
\alg & 0.01 & 0.1 & 84.04\% \\
 & 0.1 & 0.1 & \textbf{86.09}\% \\
 & 1.0 & 0.1 & 84.80\% \\
 & 10.0 & 0.01 & 79.25\% \\
\midrule
\algname{DP-SGD} \eqref{eqn:DP_biased_GD} & 0.01 & 1.0 & 51.10\% \\
 & 0.1 & 1.0 & 79.68\% \\
 & 1.0 & 1.0 & 83.89\% \\
 & 10.0 & 0.1 & 84.50\% \\
\bottomrule
\end{tabular}
\caption{Best configurations and final test accuracies of 
\algname{DP-SGD} \eqref{eqn:DP_biased_GD}  and \alg (\ref{alg:norm_norm21_dp_v1})  without server normalization
in the non-private training.}
\label{tab:norm_comparison}
\vspace{-5pt}
\end{wrapfigure}

\subsubsection{Benefits of Error Compensation} \label{app:exp_ef}

Leveraging error compensation (EC), \alg without server normalization achieves superior performance compared to \algname{DP-SGD} with direct smoothed normalization across a range of $\beta$ and $\gamma$ hyperparameter settings (where $\alpha=0.01$), in terms of the final test accuracy reported in Figure \ref{fig:ef21_effect_loss} and Table \ref{tab:norm_comparison}.
From Table \ref{tab:norm_comparison}, \alg without server normalization consistently outperforms \algname{DP-SGD} across most combinations. This trend is particularly evident for small $\beta$ values ($\beta=0.01$), where \algname{DP-SGD} achieves only 51.10\% accuracy while \alg reaches 84.04\%.
The only exception is $\beta=10.0$, where \algname{DP-SGD} outperforms \alg. However, this combination is less practical in the private setting, as too high $\beta$ values imply high private noise, thus leading to slow algorithmic convergence.  

\subsubsection{Effect of Server Normalization} \label{app:exp_serv_norm}

We investigate the impact of server-side normalization (Line 11 in Algorithm \ref{alg:norm_norm21_dp_v1}) on the convergence performance of \alg. We reported training loss and test accuracy of \alg without and with server normalization in Figure \ref{fig:serv_norm} while summarizing their final test accuracy in Table \ref{tab:serv_norm_}.

\alg without server normalization generally achieves faster convergence in training loss and higher test accuracy than \alg with server normalization across varying $\beta$ values. 
Notably, at $\beta=0.1$, \alg without server normalization achieves the highest test accuracy of \textbf{86.09\%}.
Only at the large value of $\beta=10.0$ does server normalization improve the test accuracy of \alg without server normalization by approximately $2.2\%$.
\begin{figure}[h]
\begin{center}
\centerline{\includegraphics[width=0.9\linewidth]{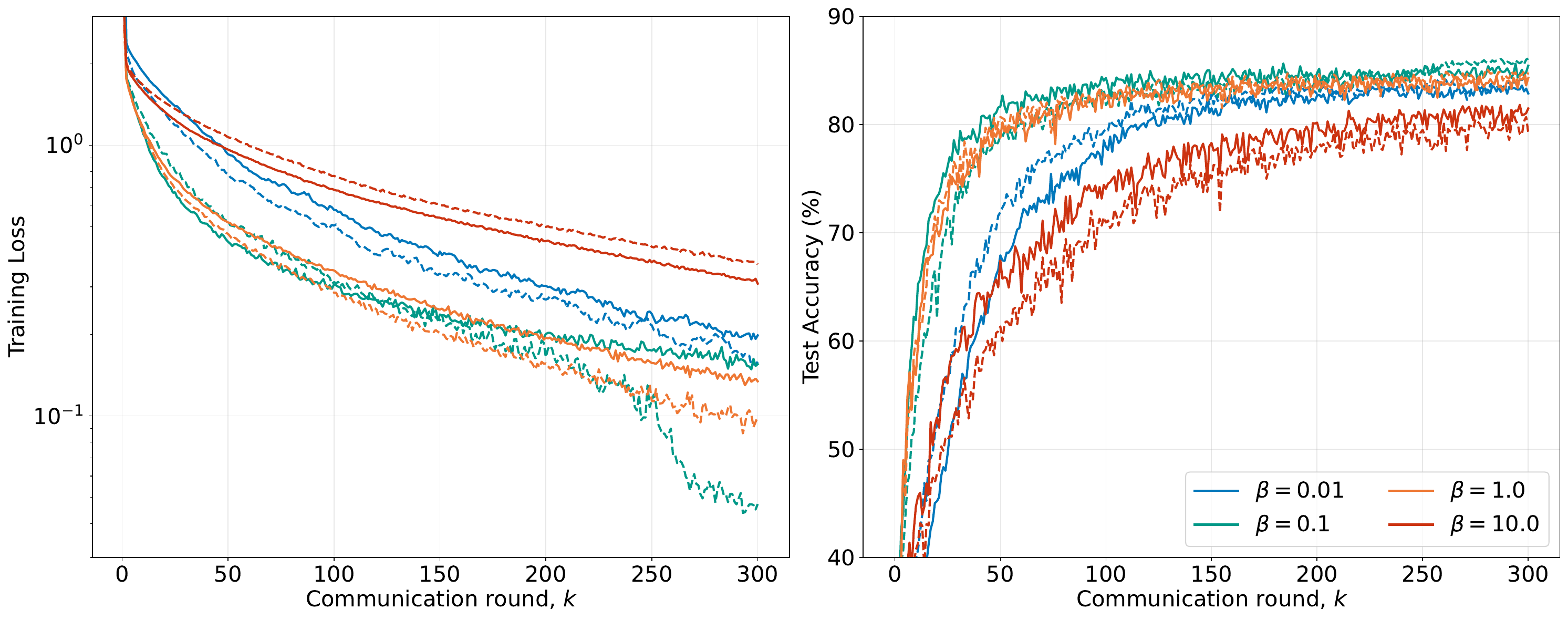}}
\caption{Training loss and test accuracy of \alg with [solid] and without [dashed] server normalization.}
\label{fig:serv_norm}
\end{center}
\vspace{-35pt}
\end{figure}

\begin{table}[h]
\centering
\begin{tabular}{cccc}
\toprule
Method: \alg & $\beta$ & $\gamma$ & Final Accuracy \\
\midrule
With server normalization & 0.01 & 0.01 & 82.86\% \\
 & 0.1 & 0.1 & 85.43\% \\
 & 1.0 & 0.1 & 84.29\% \\
 & 10.0 & 0.1 & 81.48\% \\
\midrule
Without server normalization & 0.01 & 0.1 & 84.04\% \\
 & 0.1 & 0.1 & \textbf{86.09}\% \\
 & 1.0 & 0.1 & 84.80\% \\
 & 10.0 & 0.01 & 79.25\% \\
\bottomrule
\end{tabular}
\caption{Best configurations and final test accuracies of \alg with and without server normalization.}
\label{tab:serv_norm_}
\end{table}

\subsubsection{Comparison of \algname{Clip21} and \alg} \label{app:exp_clip_norm}

\alg without server normalization\footnote{We ran \alg without server normalization because it showed better performance than \alg with server normalization according to Appendix \ref{app:exp_serv_norm}.} achieves comparable convergence performance to \algname{Clip21} for most $\beta$ values. 
At small $\beta$ values ($0.01,0.1$), \alg without server normalization attains slightly lower final test accuracy. However, at high $\beta=10.0$, \algname{Clip21} maintains the higher test accuracy, as the large clipping threshold effectively disables clipping. 
Furthermore, in most cases, both methods achieve their best performance with $\gamma=0.1$, except for \alg at $\beta=10.0$, where a smaller learning rate ($\gamma=0.01$) was optimal.

\begin{figure}[h]
\begin{center}
\centerline{\includegraphics[width=0.9\linewidth]{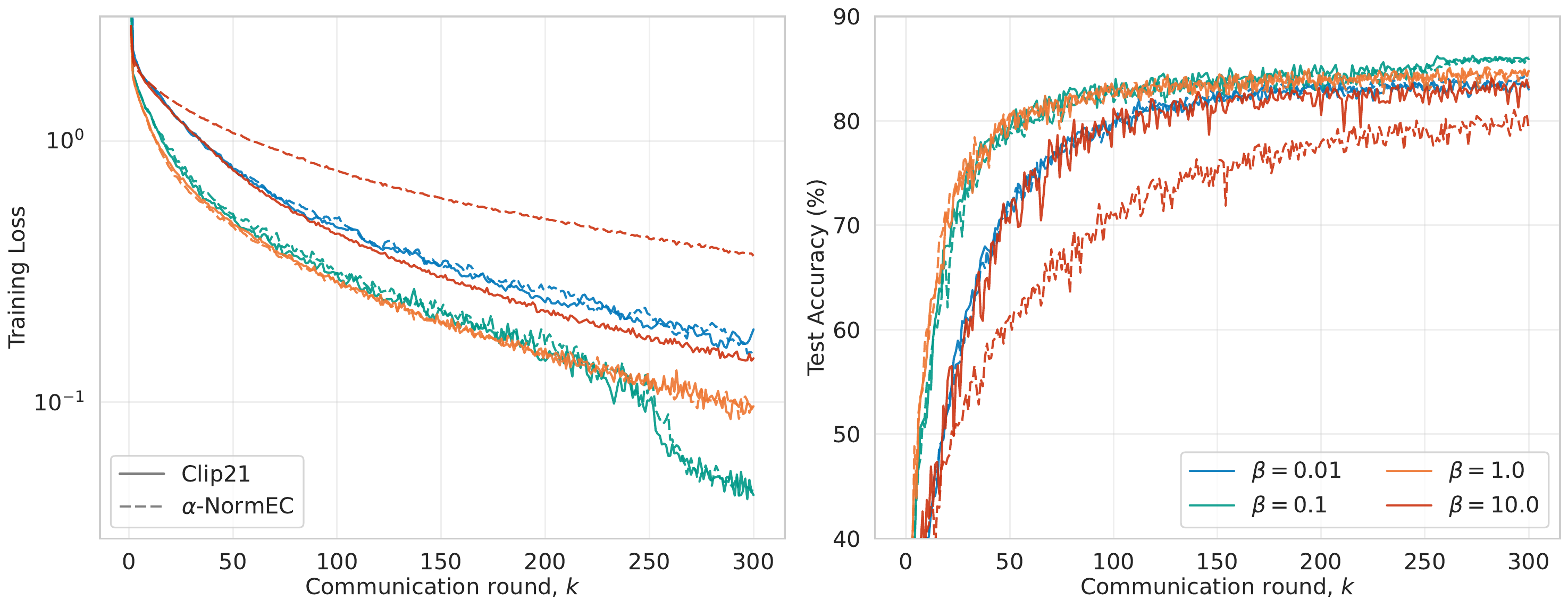}}
\caption{Training loss and test accuracy of \algname{Clip21} [solid] and \alg [dashed] without server normalization in the non-private training.}
\label{fig:clip_norm_comparison}
\end{center}
\end{figure}

\begin{figure}[h]
    \centering
    \begin{minipage}{0.48\textwidth}
        \centering
        \begin{tabular}{cccc}
            \toprule
            Method & $\beta$ & $\gamma$ & Final Accuracy \\
            \midrule
            \algname{Clip21} & 0.01 & 0.1 & 83.00\% \\
            & 0.1 & 0.1 & 85.91\% \\
            & 1.0 & 0.1 & 84.78\% \\
            & 10.0 & 0.1 & 83.19\% \\
            \midrule
            \alg & 0.01 & 0.1 & 84.04\% \\
            & 0.1 & 0.1 & \textbf{86.09}\% \\
            & 1.0 & 0.1 & 84.80\% \\
            & 10.0 & 0.01 & 79.25\% \\
            \bottomrule
        \end{tabular}
        \caption{Best configurations and final test accuracies.}
        \label{tab:clip_norm_comparison}
    \end{minipage}
    \hfill
    \begin{minipage}{0.5\textwidth}
        \centering
        \includegraphics[width=\linewidth]{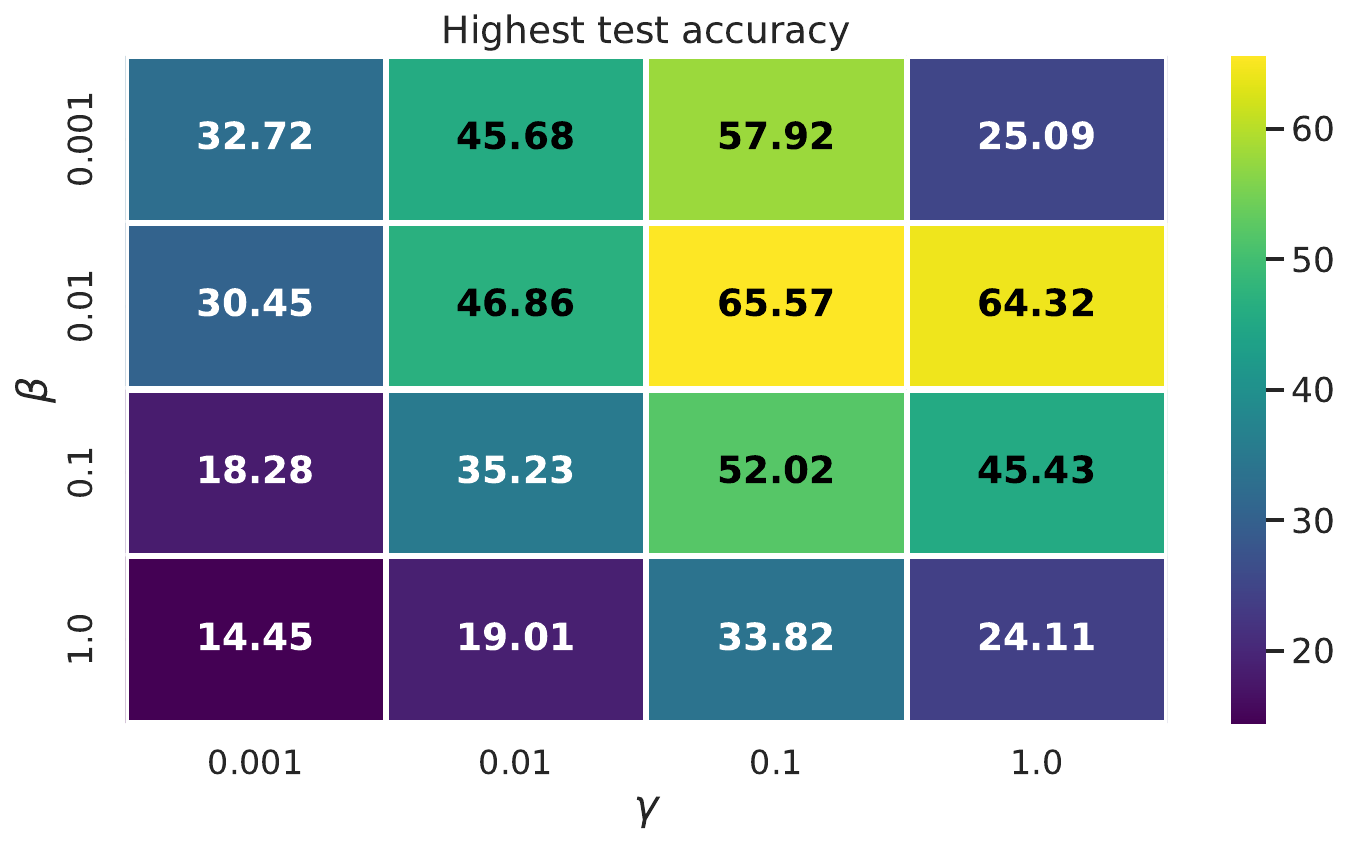}
        \caption{Best test accuracy of \algname{DP-}\alg across different $\beta, \gamma$ values.}
        \label{fig:dp_150}
    \end{minipage}
\end{figure}

\subsection{Shorter Private Training}

We present additional results in Figures \ref{fig:dp_loss} and \ref{fig:dp_150} by running \algname{DP-}\alg for 150 communication rounds. The step size $\gamma$ is tuned for every parameter $\beta$. In the non-private setting, longer training is beneficial. However, in the private scenario, it may not hold due to increased noise variance as it scales with a number of iterations. Interestingly, we observe that for $\beta=1$, the highest achieved accuracy after 150 iterations is almost the same as after a doubled communication budget of 300 iterations.

\begin{figure}[h]
\begin{minipage}[c]{0.49\textwidth}
\begin{center}
\centerline{\includegraphics[width=\linewidth]{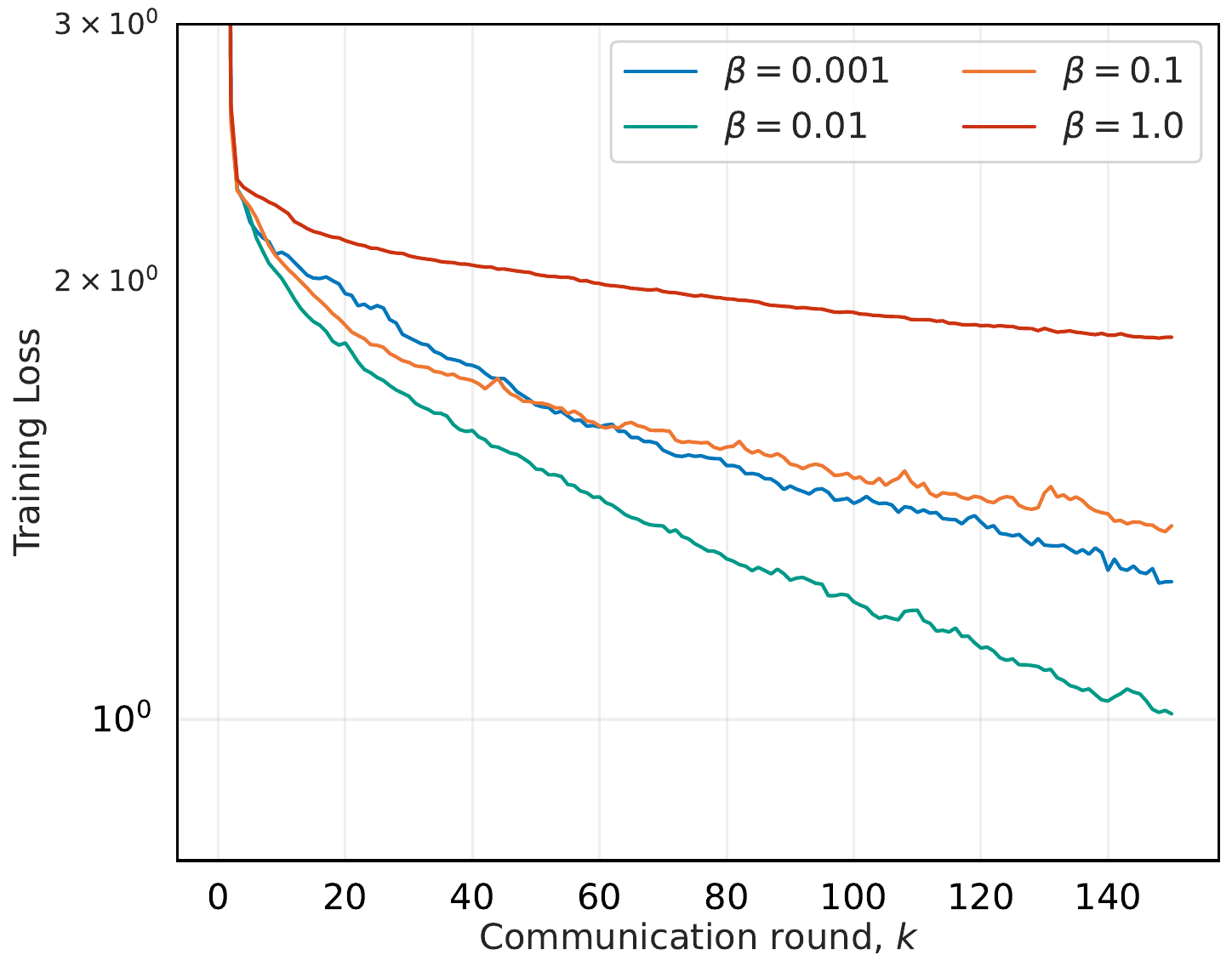}}
\end{center}
\end{minipage}
\hfill
\begin{minipage}[c]{0.49\textwidth}
\begin{center}
\centerline{\includegraphics[width=\linewidth]{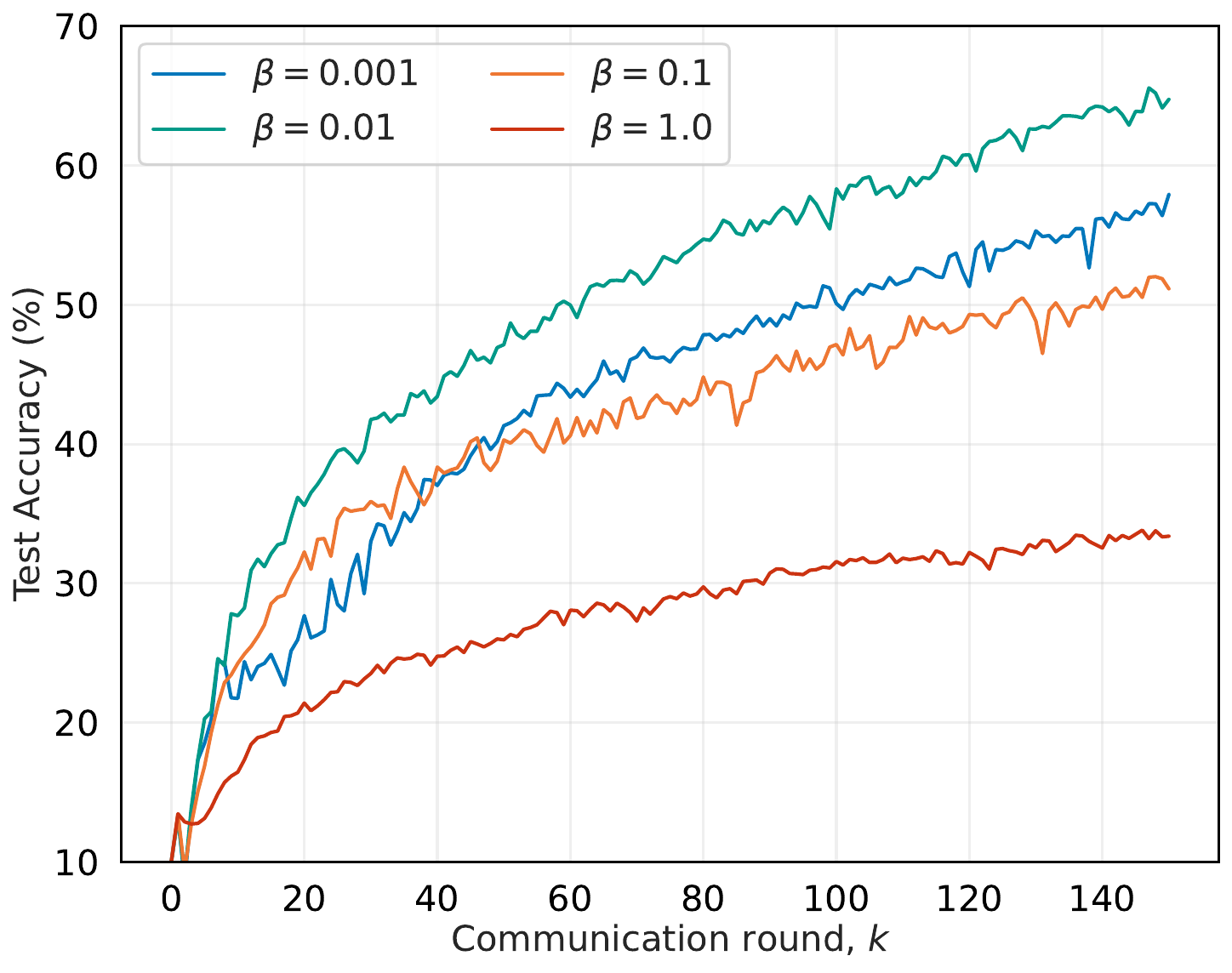}}
\end{center}
\end{minipage}
\caption{Training loss and test accuracy of \algname{DP-}\alg across different $\beta$ values.}
\label{fig:dp_loss}
\end{figure}

\end{document}